\documentclass{article}
\usepackage[utf8]{inputenc}
\usepackage{kr}

\usepackage{times}
\usepackage{soul}
\usepackage{url}
\usepackage[hidelinks,bookmarks=false]{hyperref}
\usepackage[utf8]{inputenc}
\usepackage[small]{caption}
\usepackage{graphicx}
\usepackage{amsmath}
\usepackage{amsthm}
\usepackage{booktabs}

\usepackage[textwidth=15mm,textsize=footnotesize,disable]{todonotes}
\usepackage{textcomp}
\usepackage{capt-of}
\relax

\date{}
\title{Monotonicity and Noise-Tolerance in Case-Based Reasoning \\ with Abstract Argumentation (with Appendix)
} 

\usepackage[normalem]{ulem}
\usepackage{subcaption}
\usepackage{amssymb}
\hypersetup{
pdfinfo={
Title={Monotonicity and Noise-Tolerance in Case-Based Reasoning with Abstract Argumentation},
Author={Guilherme Paulino-Passos, Francesca Toni}
}
}

\usepackage{tikz}
\usetikzlibrary{arrows,shapes}
\usepackage[boxruled,noend]{algorithm2e}
\usepackage{bm}                 \newtheorem{theorem}{Theorem}

\newtheorem{lemma}[theorem]{Lemma}
\theoremstyle{definition}
\newtheorem{definition}[theorem]{Definition}
\newtheorem{example}[theorem]{Example}
\newcommand{\defemph}[1]{\emph{#1}}

\newcommand{\gppnew}[1]{#1}
\newcommand{\gppold}[1]{\textcolor{olive}{\sout{#1}}}
\renewcommand{\gppold}[1]{}

\if@todonotes@disabled
\else
\fi

\newif\ifincludeincoherent
\includeincoherenttrue

\newif\ifincludeproofs
\includeproofsfalse

\newif\ifincludeNN              \includeNNfalse

\usepackage{xparse}
\newcommand{\citet}[2][]{\ifx&#1&\citeauthor{#2}~\shortcite{#2}\else \citeauthor{#2}~\shortcite[#1]{#2}\fi }
\newcommand{\citep}{\cite}
\newcommand{\citealp}[1]{\citeauthor{#1}~\citeyear{#1}}

\newcommand\powerset[1]{\ensuremath{2^{#1}}}
\def\wrt{w.r.t.}
\renewcommand{\phi}{\varphi} \newcommand{\card}[1]{\vert #1 \vert}
\renewcommand{\emptyset}{\varnothing}

\newcommand{\pgeq}{\succeq} \newcommand{\pg}{\succ}

\newcommand{\pleq}{\preceq}
\newcommand{\pl}{\prec}

\newcommand{\case}[2]{\mbox{\ensuremath{(#1, #2)}}}
\newcommand{\caseset}[2][]{\case{\{#1\}}{#2}}
\newcommand{\set}[1]{\mbox{\ensuremath{\{#1\}}}}
\newcommand{\defcharac}{\delta_C}
\newcommand{\defoutcome}{\delta_o}
\newcommand{\nondefoutcome}{\bar{\defoutcome}}
\newcommand{\defcase}{\case{\defcharac}{\defoutcome}}
\newcommand{\defcaseset}{\case{\emptyset}{-}}
\newcommand{\newcasearg}[1][N]{\case{{#1}_C}{?}}
\newcommand{\newcasecharac}{{N_C}}

\newcommand{\casei}{\alpha}     \newcommand{\caseii}{\beta}     \newcommand{\caseiii}{\gamma}
\newcommand{\caseiv}{\eta}
\newcommand{\casev}{\theta}
\newcommand{\charac}[1]{{#1}_C}
\newcommand{\outcome}[1]{{#1}_o}
\newcommand{\fullcase}[1]{\case{\charac{#1}}{\outcome{#1}}}

\tikzset{attack/.style={-latex}}
\newcommand\AF[2]{\case{#1}{#2}}
\def\myAF{\ensuremath{(\Args,\attacks)}}
\def\myAFalone{\ensuremath{(\Args',\attacks')}}
\def\Args{\ensuremath{\mathit{Args}}}
\def\attacks{\ensuremath{\leadsto}}
\def\groundext{\mathbb{G}}
\def\arga{\ensuremath{\alpha}}
\def\argb{\ensuremath{\beta}}
\def\argc{\ensuremath{\gamma}}

\def\coherent{coherent}
\def\coherence{coherence}
\def\Coherence{Coherence}
\newcommand{\wellbehaved}{regular} \newcommand{\concision}{requirement 3 in the second bullet of Definition~\ref{def:dear-aacbr-def}}

\def\aaD{\ensuremath{AF_{\pgeq}(D)}} \def\aaDN{\ensuremath{AF_{\pgeq}(D,\newcasecharac)}}
 \def\caaDN{\ensuremath{cAF_{\pgeq}(D,\newcasecharac)}}
\newcommand{\aaFtwo}[2]{\mbox{\ensuremath{AF_{\pgeq}(#1, #2)}}} \newcommand{\aaFone}[1]{\mbox{\ensuremath{AF_{\pgeq}(#1)}}}

\mathchardef\mhyphen="2D

\newcommand{\aacbr}{\ensuremath{{AA\mhyphen CBR}}} \newcommand{\AACBR}{\aacbr} 
  \newcommand{\paacbr}{\ensuremath{\aacbr_{\pgeq}}} 
\newcommand{\pAACBR}{\paacbr}
\newcommand{\PAACBR}{\paacbr}
\newcommand{\oaacbr}{\ensuremath{\aacbr_{\supseteq}}}  \newcommand{\caacbr}{\ensuremath{c{\paacbr}}} \newcommand{\cAACBR}{\caacbr}
\newcommand{\inferspaacbr}{\vdash_{\paacbr}}

\newcommand{\cinfers}{\vdash_{\learn}}

\def\includable{includable}     \def\includability{includability} \def\surprising{surprising}
\def\surprise{surprise}
\def\sufficient{sufficient}

\newcommand{\lang}{\mathcal{L}}
\newcommand{\learn}{\mathbb{C}}
\newcommand{\infers}{\vdash}

\newcommand{\defendant}{\Delta}
\newcommand{\plaintiff}{\Pi}

 \author{Guilherme Paulino-Passos$^1$\and
  Francesca Toni$^1$

  \affiliations{$^1$Imperial College London, Department of Computing \\  
  \emails{\{g.passos18, f.toni\}@imperial.ac.uk}\\
}
}

\begin{document}

\maketitle
\begin{abstract}
  Recently, abstract argumentation-based models of case-based reasoning ($\aacbr$ in short) have been proposed, originally inspired by the legal domain, but also applicable as classifiers in different scenarios. However, the formal properties of $\aacbr$ as a reasoning system remain largely unexplored. In this paper, we focus on analysing the non-monotonicity properties of a \emph{\wellbehaved} version of $\aacbr$ (that we call $\paacbr$). Specifically, we prove that $\paacbr$ is not cautiously monotonic, a property frequently considered desirable in the literature. We then define a variation of $\paacbr$ which is cautiously monotonic. Further, we prove that such variation is equivalent to using $\paacbr$ with a restricted casebase consisting of all ``{\surprising}'' and ``\sufficient'' cases in the original casebase. As a by-product, we prove that this variation of $\paacbr$ is cumulative, rationally monotonic, and empowers a principled treatment of noise in ``incoherent'' casebases. Finally, we illustrate $\aacbr$ and cautious monotonicity questions on a case study on the U.S. Trade Secrets domain, a legal casebase.
\end{abstract}

\section{Introduction}
\label{sec:orgb98c4b7}

\emph{Case-based reasoning (CBR)} relies upon known solutions for problems (past cases) to infer solutions for unseen problems (new cases), based upon retrieving past cases which are ``similar'' to the new cases. 
It  is widely used in legal settings (e.g. \citealp{trevor}; \citealp{DBLP:conf/kr/CyrasST16}), for classification (e.g. via the k-NN algorithm and, recently, within the DEAr methodology \cite{dear-2020}) and for explanation (e.g. see
\citealp{DBLP:journals/air/NugentC05}; \citealp{DBLP:conf/ijcai/KennyK19}; \citealp{dear-2020}).

In this paper we focus on a recent approach to CBR based upon an argumentative reading of (past and new) cases  
\cite{DBLP:conf/kr/CyrasST16,DBLP:conf/comma/CyrasST16,Cocarascu:2018,DBLP:journals/eswa/CyrasBGTDTGH19,dear-2020}, and using 
\emph{Abstract Argumentation (AA)} \cite{Dung:95} as the underpinning machinery.
We will refer to all proposed incarnations of this approach in the literature generically as \AACBR\ (the acronym used in the original paper \cite{DBLP:conf/kr/CyrasST16}): they all generate an AA framework from a CBR problem: a graph structure where cases are arguments, ``more specific'' past cases attack ``less specific'' past cases or a ``default argument'' (which embeds a sort of bias), and new cases attack ``irrelevant'' past cases; then, CBR is reduced to testing membership of this default argument in the grounded extension \cite{Dung:95}. The use of argumentation in \AACBR\ naturally paves the way towards explanation generation
for CBR tasks, e.g. in the form of dispute trees in \cite{DBLP:conf/comma/CyrasST16,dear-2020} or excess features in \cite{DBLP:journals/eswa/CyrasBGTDTGH19}, possibly 
for supporting interactions with users and CBR, building upon recent research on incorporating feedback in recommender systems \cite{DBLP:journals/ai/RagoCBLT21} and showing influence structures from neural network classifiers \cite{DBLP:conf/atal/DejlHMMSVAL0T21}.

Different incarnations of \aacbr\  use different mechanisms for defining the aforementioned ``specificity'', ``irrelevance'' and ``default argument'': the original version in \cite{DBLP:conf/kr/CyrasST16} is applicable only to cases characterised by sets of features and defines all three notions in terms of subsets, while the version used for classification in \cite{dear-2020} defines specificity in terms of a generic partial order, irrelevance in terms of a generic relation and default argument in terms of a generic characterisation. Thus, it is in principle applicable to cases characterised in any way, as sets of features or unstructured \cite{dear-2020}.
We will study a special,  \emph{\wellbehaved} instance ($\paacbr$) of this more recent presentation, 
in which irrelevance and the default argument are both defined 
via specificity (and in particular the default argument is defined in terms of the most specific case). 
\paacbr\ admits the original \aacbr{} in \cite{DBLP:conf/kr/CyrasST16} as an instance, obtained by choosing the partial order to be the subset relation and by restricting attention to ``\coherent'' casebases 
(whereby there is no ``noise'', in that no two cases with different outcomes are characterised by the same set of features).

\aacbr\ was originally inspired by the legal domain in \cite{DBLP:conf/kr/CyrasST16}, but {some incarnations of \aacbr, integrating dynamic features, have proven useful in predicting and explaining} the passage of bills in the UK Parliament \cite{DBLP:journals/eswa/CyrasBGTDTGH19}, and 
instantiations of the more generic version of \citet{dear-2020} have shown to be fruitfully applicable as classifiers 
\cite{dear-2020}.
We study \emph{non-monotonicity} properties of \paacbr\ understood at the same time as a reasoning system and as a classifier. 
These properties, typically considered in logic,  intuitively characterise in which sense systems may stop inferring some conclusions when more information becomes available \cite{generalpatterns}.
These properties are thus related to modelling inference which is tentative\ {and} defeasible, as opposed to the indefeasible form of inference of classical logic. Non-monotonicity properties have already been studied in argumentation, e.g. for ABA, ABA+ \cite{DBLP:conf/tafa/CyrasT15,DBLP:journals/corr/CyrasT16}, $ASPIC^{+}$ \cite{DBLP:conf/ecai/Dung14,Dung_2016} and logic-based argumentation \cite{DBLP:conf/comma/Hunter10}. {We study them for the application of argumentation to classification via {\paacbr}.}

Specifically, we prove that the kind of inference underpinning \paacbr\ lacks a standard non-monotonicity property, namely \emph{cautious monotonicity}, sanctioning, intuitively, that if a conclusion is added to the set of premises (here, the casebase), then no conclusion is lost, that is, everything which was inferable still is so. In terms of a {supervised} classifier, satisfying cautious monotonicity culminates in being ``closed'' under self-training. That is, augmenting the dataset with conclusions inferred by the classifier itself does not change the classifier.
Then, we make a two-fold contribution: we define (formally and algorithmically) 
a provably cautiously monotonic variant of \paacbr, that we call \caacbr, and prove that  it is equivalent to \paacbr\ applied to a restricted casebase consisting of all ``{\surprising}'' and ``{\sufficient}'' cases in the original casebase.
We also show that cautious monotonicity of \caacbr\ {leads to} the desirable properties of \emph{cumulativity} and \emph{rational monotonicity}, and that, as a by-product, our cautiously monotonic variant leads to a desirable treatment of noise in ``incoherent'' casebases,
 including cases with the same set of features but different outcomes.
Incoherence may result from a limited language to express features or genuine errors in generating the casebases. Independently of the reasons behind incoherence, and especially when this is outside the control of reasoning system designers, it is important that the reasoning system is able to tolerate it. It is interesting that \caacbr\ deals with it serendipitously, as a direct consequence of insuring cautious monotonicity.

This paper generalises our previous work \cite{caacbr} by also dealing with in\coherent\ casebases, presenting a case study, and discussing the position of our contribution in the related literature.
We omit some proofs in the main text for lack of space but they are available in the appendix.

\section{Motivating illustration}
\label{sec:org9dc41b1}

In this section we introduce a simple setting for the informal illustration 
of the original \AACBR, its non-monotonicity  
and the desirability of some restrictions thereof, as well as problems raised by the presence of in\coherence\ in the case base when deploying \AACBR. Thus, this section serves as a motivating illustration for our approach, which restricts non-monotonicity and is in\coherence-tolerant.

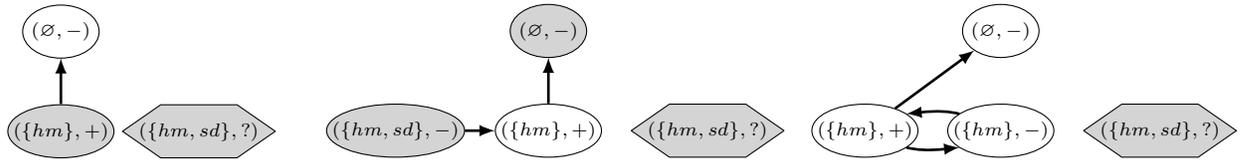
\begin{figure*}[t!hb]
  \centering
\begin{subfigure}[t]{0.25\textwidth}
  \centering
  ~
   \begin{tikzpicture}[>=latex,line join=bevel, scale=0.8, font=\scriptsize]
     \pgfsetlinewidth{1bp}
\pgfsetcolor{black}
\draw [->] (25.0bp,25.343bp) .. controls (25.0bp,28.924bp) and (25.0bp,32.924bp)  .. (25.0bp,46.997bp);
\begin{scope}
     \definecolor{strokecol}{rgb}{0.0,0.0,0.0};
     \pgfsetstrokecolor{strokecol}
     \definecolor{fillcol}{rgb}{0.83,0.83,0.83};
     \pgfsetfillcolor{fillcol}
     \filldraw [thin] (126.0bp,12.5bp) -- (108.0bp,25.0bp) -- (72.0bp,25.0bp) -- (54.0bp,12.5bp) -- (72.0bp,0.0bp) -- (108.0bp,0.0bp) -- cycle;
     \draw (90.0bp,12.5bp) node {$\caseset[hm,sd]{?}$};
   \end{scope}
\begin{scope}
     \definecolor{strokecol}{rgb}{0.0,0.0,0.0};
     \pgfsetstrokecolor{strokecol}
     \definecolor{fillcol}{rgb}{0.83,0.83,0.83};
     \pgfsetfillcolor{fillcol}
     \filldraw [opacity=1] [thin] (25.0bp,12.5bp) ellipse (25.0bp and 12.5bp);
     \draw (25.0bp,12.5bp) node {$\caseset[hm]{+}$};
   \end{scope}
\begin{scope}
     \definecolor{strokecol}{rgb}{0.0,0.0,0.0};
     \pgfsetstrokecolor{strokecol}
     \draw [thin] (25.0bp,59.5bp) ellipse (18.0bp and 12.5bp);
     \draw (25.0bp,59.5bp) node {$\defcaseset$};
   \end{scope}
\end{tikzpicture}
   \caption{{Initial AA} framework. {\AACBR\ predicts outcome ``$+$'' for the new case}.}
  \label{fig:example-legal-1}
\end{subfigure}
~
\begin{subfigure}[t]{0.35\textwidth}
  \centering
   \begin{tikzpicture}[>=latex,line join=bevel, scale=0.8, font=\scriptsize]
  \pgfsetlinewidth{1bp}
\begin{scope}
  \pgfsetstrokecolor{black}
  \definecolor{strokecol}{rgb}{1.0,1.0,1.0};
  \pgfsetstrokecolor{strokecol}
  \definecolor{fillcol}{rgb}{1.0,1.0,1.0};
  \pgfsetfillcolor{fillcol}
  \filldraw (0.0bp,0.0bp) -- (0.0bp,72.0bp) -- (215.5bp,72.0bp) -- (215.5bp,0.0bp) -- cycle;
\end{scope}
  \pgfsetcolor{black}
\draw [->] (104.5bp,25.343bp) .. controls (104.5bp,28.924bp) and (104.5bp,32.924bp)  .. (104.5bp,46.997bp);
\draw [->] (65.125bp,12.5bp) .. controls (66.452bp,12.5bp) and (67.778bp,12.5bp)  .. (79.275bp,12.5bp);
\begin{scope}
  \definecolor{strokecol}{rgb}{0.0,0.0,0.0};
  \pgfsetstrokecolor{strokecol}
  \draw [thin] (104.5bp,12.5bp) ellipse (25.0bp and 12.5bp);
  \draw (104.5bp,12.5bp) node {$\caseset[hm]{+}$};
\end{scope}
\begin{scope}
  \definecolor{strokecol}{rgb}{0.0,0.0,0.0};
  \pgfsetstrokecolor{strokecol}
  \definecolor{fillcol}{rgb}{0.83,0.83,0.83};
  \pgfsetfillcolor{fillcol}
  \filldraw [opacity=1] [thin] (32.5bp,12.5bp) ellipse (32.5bp and 12.5bp);
  \draw (32.5bp,12.5bp) node {$\caseset[hm,sd]{-}$};
\end{scope}
\begin{scope}
  \definecolor{strokecol}{rgb}{0.0,0.0,0.0};
  \pgfsetstrokecolor{strokecol}
  \definecolor{fillcol}{rgb}{0.83,0.83,0.83};
  \pgfsetfillcolor{fillcol}
  \filldraw [thin] (215.5bp,12.5bp) -- (197.5bp,25.0bp) -- (161.5bp,25.0bp) -- (143.5bp,12.5bp) -- (161.5bp,0.0bp) -- (197.5bp,0.0bp) -- cycle;
  \draw (179.5bp,12.5bp) node {$\caseset[hm,sd]{?}$};
\end{scope}
\begin{scope}
  \definecolor{strokecol}{rgb}{0.0,0.0,0.0};
  \pgfsetstrokecolor{strokecol}
  \definecolor{fillcol}{rgb}{0.83,0.83,0.83};
  \pgfsetfillcolor{fillcol}
  \filldraw [opacity=1] [thin] (104.5bp,59.5bp) ellipse (18.0bp and 12.5bp);
  \draw (104.5bp,59.5bp) node {$\defcaseset$};
\end{scope}
\end{tikzpicture}
   \caption{{Revised} framework. The added past case changes the \AACBR-predicted outcome to ``$-$''.}
  \label{fig:example-legal-2}
\end{subfigure}
~ 
\begin{subfigure}[t]{0.3\textwidth}
  \centering
\begin{tikzpicture}[>=latex,line join=bevel, scale=0.8, font=\scriptsize]
  \pgfsetlinewidth{1bp}
\begin{scope}
  \pgfsetstrokecolor{black}
  \definecolor{strokecol}{rgb}{1.0,1.0,1.0};
  \pgfsetstrokecolor{strokecol}
  \definecolor{fillcol}{rgb}{1.0,1.0,1.0};
  \pgfsetfillcolor{fillcol}
  \filldraw (0.0bp,0.0bp) -- (0.0bp,72.0bp) -- (200.0bp,72.0bp) -- (200.0bp,0.0bp) -- cycle;
\end{scope}
  \pgfsetcolor{black}
\draw [->] (39.204bp,22.931bp) .. controls (47.775bp,29.225bp) and (58.844bp,37.354bp)  .. (76.565bp,50.368bp);
\draw [->] (44.5bp,4.5555bp) .. controls (49.454bp,3.4216bp) and (54.408bp,2.9618bp)  .. (69.531bp,4.5627bp);
\draw [->] (69.531bp,20.437bp) .. controls (64.577bp,21.574bp) and (59.623bp,22.037bp)  .. (44.5bp,20.444bp);
\begin{scope}
  \definecolor{strokecol}{rgb}{0.0,0.0,0.0};
  \pgfsetstrokecolor{strokecol}
  \draw [thin] (25.0bp,12.5bp) ellipse (25.0bp and 12.5bp);
  \draw (25.0bp,12.5bp) node {$\caseset[hm]{+}$};
\end{scope}
\begin{scope}
  \definecolor{strokecol}{rgb}{0.0,0.0,0.0};
  \pgfsetstrokecolor{strokecol}
  \draw [thin] (89.0bp,12.5bp) ellipse (25.0bp and 12.5bp);
  \draw (89.0bp,12.5bp) node {$\caseset[hm]{-}$};
\end{scope}
\begin{scope}
  \definecolor{strokecol}{rgb}{0.0,0.0,0.0};
  \pgfsetstrokecolor{strokecol}
  \draw [thin] (89.0bp,59.5bp) ellipse (18.0bp and 12.5bp);
  \draw (89.0bp,59.5bp) node {$\defcaseset$};
\end{scope}
\begin{scope}
  \definecolor{strokecol}{rgb}{0.0,0.0,0.0};
  \pgfsetstrokecolor{strokecol}
  \definecolor{fillcol}{rgb}{0.83,0.83,0.83};
  \pgfsetfillcolor{fillcol}
  \filldraw [thin] (200.0bp,12.5bp) -- (182.0bp,25.0bp) -- (146.0bp,25.0bp) -- (128.0bp,12.5bp) -- (146.0bp,0.0bp) -- (182.0bp,0.0bp) -- cycle;
  \draw (164.0bp,12.5bp) node {$\caseset[hm,sd]{?}$};
\end{scope}
\end{tikzpicture}
\caption{In\coherent{} casebase. \AACBR\ predicts outcome ``$+$'', but the default argument is not attacked by arguments in $\groundext$.} 
  \label{fig:example-legal-3}
\end{subfigure}
\caption{AA frameworks when using \AACBR\ for Examples \ref{running-example} and \ref{ex:noise}. Past cases {(with outcomes)} {and} new case {(with unknown outcome)} are arguments. (Grounded extensions $\groundext$ are shaded.)}
\label{fig:example-legal-intro}
\end{figure*}

\begin{example}[Non-Monotonicity]
\label{running-example}
Consider a simplified legal system built by cases and adhering, like most modern legal systems, to the principle by which, unless proven otherwise, no person is to be considered guilty of a crime. This can be represented by a ``default argument'' \(\defcaseset\), indicating that, in the absence of any information about any person, the legal system should infer a negative outcome  $-$ (that the person is \emph{not} guilty). {\(\defcaseset\) can be understood as an argument, in the AA sense, given that it} is merely what is called a relative presumption, since it is open to proof to the contrary{, e.g.} by proving that the person did indeed commit a crime. Let us consider here one possible crime: homicide\footnote{This
		is a toy example, so the terms used do not correspond to a specific jurisdiction.} (\emph{hm}). In one case, it was established that the defendant committed homicide, and he was considered guilty, represented as \(\caseset[hm]{+}\).

	{Consider now a new case \(\caseset[hm,sd]{?}\){, with an unknown outcome,} of a defendant who committed homicide, but for which it was proven that it was in self-defence (\emph{sd}). {In order to predict the new case's} outcome by CBR, $\aacbr$ reduces {the} prediction problem to that of membership of the default argument in the grounded extension~\cite{Dung:95} $\groundext$ of the AA framework} 
	in Figure \ref{fig:example-legal-1}: given that \(\defcaseset\not\in\groundext\), the predicted outcome is positive (i.e. guilty), disregarding $sd$ and, indeed, no matter what other feature this case may have. Thus, up to this point, having the feature $hm$ is a sufficient condition for {predicting} guilty.
If, however, the courts decide that for this new case the defendant should be acquitted, the case \(\caseset[hm,sd]{-}\) enters in our casebase. Now, having the feature $hm$ is {no longer} a sufficient condition for {predicting} guilty, {and} any case {with} both $hm$ and $sd$ will {be predicted} a negative  outcome (i.e. {that the person is} innocent). {This is the case for predicting the outcome of a new case with again  both $hm$ and $sd$, in \AACBR\ using the AA framework}  
in Figure \ref{fig:example-legal-2}.  Thus, adding a new case to the casebase removed some conclusions which were inferred from the previous, smaller casebase, showing that \AACBR\ is indeed non-monotonic. This does not mean that some restrictions on non-monotonicity might not be desirable. For instance, we might expect in a legal system that, if for the current case law, two cases are to be judged in a certain way, then one of the cases happening in court and indeed being decided in that way would not affect the body of case law itself, thus the outcome for the second case would be expected to be unchanged.
\end{example}

The following example illustrates the challenges posed by in\coherent\ casebases, with noisy cases, in \aacbr.

\begin{example}[Noise-intolerance]
	\label{ex:noise}
	Consider a different augmentation of the initial casebase in Example~\ref{running-example}, resulting in the casebase \(\{\caseset[hm]{+}, \caseset[hm]{-}\}\), whereby a positive and a negative outcomes are recorded for exactly the same profile for defendants. This can be deemed to be ``incoherent''\footnote{This casebase may result from a limited language for characterising cases, e.g. ignoring the possibility of indicating core differences between defendants, such as that the defendant characterised by $ \caseset[hm]{-}$ is a minor. We assume here that this incoherence cannot be rectified by a language variation, or simply that it comes from the data, and we cannot remove it based on the data alone.}. In \aacbr, this case would have a positive outcome (guilty), since the default argument is not in the grounded extension. However, this seems unsatisfactory, since the default argument is not attacked by any argument in the grounded extension of the corresponding AA framework (see Figure~\ref{fig:example-legal-3}). In some sense, this means that every past case was rejected, but the outcome is guilty nonetheless. Thus what is deciding the outcome is not a relevant case per se, but the incoherence itself. This can be regarded as a form of intolerance to noise. Note that incoherence does not always give this form of noise intolerance: if a past case \(\caseset[hm,sd]{+}\) were included, the outcome would be the same, but the default argument would be attacked by (some argument in) the grounded extension. 

	Our form of \aacbr\ is guaranteed to be tolerant of noise always, independently of the casebase.
\end{example}

\section{Preliminaries}
\label{sec:org43bfd8b}
\paragraph{{\bf Abstract argumentation.}}
\label{sec:orgd45d2e6}
An \emph{abstract argumentation framework (AF)} \cite{Dung:95} is a pair $\myAF$,
where $\Args$ is a set (of \emph{arguments}) 
and $\attacks$ is a binary relation on $\Args$. 
For $\arga, \argb \in \Args$, if $\arga \attacks \argb$, 
then we say that $\arga$ \emph{attacks} $\argb$
and that $\arga$ is an \emph{attacker of} $\argb$. 
For a set of arguments $E \subseteq \Args$ and an argument $\arga \in \Args$, 
$E$ \emph{defends} $\arga$ if for all $\argb \attacks \arga$ there exists $\argc \in E$ such that $\argc \attacks \argb$. 
Then, 
\label{defn:semantics} 
the \emph{grounded extension} of $\myAF$ can be constructed as 
$\groundext = \bigcup_{i \geqslant 0} G_i$, 
where $G_0$ is the set of all unattacked arguments, 
and $\forall i \geqslant 0$, $G_{i+1}$ is the set of arguments that $G_i$ defends.
For any $\myAF$, 
the grounded extension $\groundext$ 
always exists and is unique
and, if $\myAF$ is well-founded \cite{Dung:95}, extensions under other semantics (e.g. {stable extensions}) are equal to $\groundext$. 
In particular for finite AFs, $\myAF$ is well-founded iff it is acyclic.
{Given $\myAF$, we will sometimes use $\arga \in \myAF$ to stand for $\arga \in \Args$.}
\paragraph{{\bf Non-monotonicity properties.}}
\label{sec:non-monot-prop}

We will be interested in the following properties.\footnote{We are mostly following the treatment of \cite{generalpatterns}.}
An arbitrary inference relation $\infers$ (for a language including, in particular, sentences $a, b$, etc., {with negations $\neg a$ and $\neg b$, etc.}, and sets of sentences $A,B$)
is said to satisfy: 
\begin{enumerate}
  \item {\em non-monotonicity}, iff $A \infers a$ and $A \subseteq B$ do not imply that $B \infers a$;
\item \emph{cautious monotonicity}, iff $A \infers a$ and $A \infers b$ imply that $A \cup \{a\} \infers b$;
  \item \emph{cut}, iff $A \infers a$ and $A \cup \{a\} \infers b$ imply that $A \infers b$;
  \item \emph{cumulativity}, iff $\vdash$ is both cautiously monotonic and satisfies cut;
  \item {\em rational monotonicity}, iff $A \infers a$ and $A \not\infers \neg b$ imply that $A \cup \{b\} \infers a$; \item {\em completeness}, iff either $A \infers a$ or $A \infers \neg a$.
\end{enumerate}

\section{Abstract argumentation for case-based reasoning}
\label{sec:aacbr}

\label{sec:org3dbe745}

Here, we define \paacbr, adapting definitions from \cite{dear-2020}.
All incarnations of \aacbr, including \paacbr, 
map a {\emph{dataset}} \(D\) of {{\em examples}} labelled with an {\em outcome} and an {\em {unlabelled example}} (with unknown outcome) into an AF. {The dataset may be understood as a {\em casebase}, the labelled examples as {\em past cases} and the unlabelled example as a {\em new case}: we will use these terminologies interchangeably throughout.}
In this paper, as in \cite{dear-2020},
{examples/}cases have a characterisation (e.g., as in \cite{DBLP:conf/kr/CyrasST16}, characterisations may be sets of features), and outcomes are chosen from two available ones, one of which is selected up-front as the \emph{default outcome}.
Finally, in the spirit of \cite{dear-2020}, we assume that the set of characterisations of (past and new) cases  is equipped with a partial order {$\pleq$} (whereby $\arga \!\prec \!\argb$ {holds if $\arga \!\pleq \!\argb$ and $\arga \!\neq \!\argb$ and} is read ``$\arga$ is less \emph{specific} than $\argb$'') and with a relation $\not \sim$  (whereby $\arga \!\not\sim \!\argb$ is read as ``$\argb$ is {\em irrelevant} to $\arga$'').
Formally: 

\begin{definition}[Adapted from \cite{dear-2020}]
  \label{dear-miner}
  Let $X$ be a set of \emph{characterisations}, equipped with partial order $\pl$ and binary relation $\not\sim$.  Let $Y \!= \!\{\defoutcome,\nondefoutcome\}$ be the set of (all possible) \emph{outcomes}, with $\defoutcome$ the {\em default outcome}.  
Then, a {\em casebase} $D$ is a finite set such that  $D \!\subseteq \!X \!\times \!Y$
	(thus a {\em past case} $\alpha\in D$   is of the form $\case{\alpha_{C}}{\alpha_{o}}$ for $\alpha_{C}\!\in \!X$, $\alpha_{o}\!\in \!Y$)
        and a {\em new case}  is of the form $\newcasearg$  for $\newcasecharac \!\in \!X$.
        {We also discriminate a particular element $\defcharac \!\in \!X$ and define the \emph{default argument} $\defcase \!\in \!X \!\times \!Y$.}

	A casebase $D$ is {\em \coherent} if there are no two cases  $\case{\alpha_{C}}{\alpha_{o}},\case{\beta_{C}}{\beta_{o}}\in D$ such that $\alpha_{C} = \beta_{C}$ but $\alpha_{o} \neq \beta_{o}$, and it is \emph{in\coherent} otherwise.
\end{definition}
For simplicity of notation, we sometimes extend the definition of $\pgeq$ to $X \times Y$, by setting $\case{\alpha_c}{\alpha_o} \pgeq \case{\beta_c}{\beta_o}$ iff $\alpha_c \pgeq \beta_c$.\footnote{In \cite{dear-2020}, $\pgeq$ was directly given over $X\times Y$. Note that, 
when $D$ is \coherent, our ``lifted'' $\pgeq$ is guaranteed to be a partial order on $X \!\times \!Y$ (and thus equivalent to the one in \cite{dear-2020}), but when $D$ is in\coherent{} anti-symmetry may fail for two cases with different outcomes but same characterisation, and thus $\pgeq$ is merely a preorder on $X \times Y$. 
} 

\begin{definition} [Adapted from \cite{dear-2020}] \label{def:dear-aacbr-def}
  The \emph{AF mined from a dataset $D$ and a new case $\newcasearg$} is $\myAF$,
  in which:
  \begin{itemize}
  \item $\Args=  D \cup \{\defcase\} \cup \{\newcasearg\}$ ;    
  \item  for $(\alpha_C, \alpha_o), (\beta_C, \beta_o) \in D \cup \{ \defcase \}$, it holds that $(\alpha_C, \alpha_o) \attacks (\beta_C, \beta_o)$ iff

    \begin{enumerate}
\item $\alpha_o \neq \beta_o$,

    \item {$\alpha_C  \pgeq \beta_C$, and}
      
    \item {$\nexists (\gamma_C, \gamma_o) \in D\cup \{ \defcase \}$ with $\alpha_C \pg \gamma_C \pg \beta_C$} and {$\gamma_o = \alpha_o$};

\end{enumerate}

  \item  for $(\beta_C, \!\beta_o) \!\!\in \!\!D \cup \!\{ \!\defcase \!\}$, it holds that $\case{\newcasecharac}{?} \!\!\attacks \!\!(\beta_C, \!\beta_o)$ iff 
    $\newcasearg \!\!\not \sim \!\!(\beta_C,\!\beta_o)$.
  \end{itemize}
	{The \emph{AF mined from a dataset $D$ alone} is $\myAFalone$, 
	with 
  $\Args'=  \Args \setminus \{\newcasearg\}$ and
	$\attacks' =\attacks \cap (\Args'\times \Args')$.}
\end{definition}
Note that if $D$ is \coherent, then the ``equals'' case in item 2 of the definition of attack will never apply. As a result, the AF mined from a \coherent\ $D$ (and any $\newcasearg$) is guaranteed to be well-founded, in the sense of \citet{Dung:95}.

\begin{definition}[Adapted from \cite{dear-2020}]
	Let $\groundext$ be the grounded extension of the AF mined from $D$ and $\newcasearg$, with default argument $\defcase$.  
  The \defemph{outcome} \defemph{for $\newcasecharac$} is $\defoutcome$ if $\defcase$ is in $\groundext$, and $\nondefoutcome$ otherwise. 
      \end{definition}
In this paper we focus on {a particular case of this scenario}:
\begin{definition} \label{def:wellbehav}
  The AF mined from $D$ alone and the AF mined from $D$ and $\newcasearg$,
	with default argument $\defcase$, are \defemph{\wellbehaved} when the following holds:
        \begin{enumerate}
\item the irrelevance relation $\not\sim$ is defined as: $x_1 \not \sim x_2$ iff $x_1 \not \pgeq x_2$, and
\item $\defcharac$ is the least element of $X$.\footnote{Indeed this is not a strong condition, since it can be proved that if $\charac\casei \not \pgeq \defcharac$ then all cases $\fullcase\casei$ in the casebase could be removed, as they would never change an outcome. On the other hand, assuming also the first condition in Definition \ref{def:wellbehav}, if $\case{\charac\casei}{?}$ is the new case and $\charac\casei \not \pgeq \defcharac$, then the outcome  is $\nondefoutcome$ necessarily. }

\end{enumerate}
\end{definition}
{This restriction connects the treatment of a characterisation $\charac\casei$ as a new case and as a past case \ifincludeNN\else 
and is necessary in order to satisfy desirable properties, such as a relation between new cases and ``nearest'' past cases, omitted here.\fi}

From now on,  we will restrict our attention to \wellbehaved\ mined AFs. We will refer to the (\wellbehaved) AF mined from  $D$ and $\newcasearg$, with  default argument $\defcase$, as  
\aaDN, and to the (\wellbehaved) AF mined from  $D$ alone as \aaD.
Also, for short, given $\aaDN$,  with default argument $\defcase$,
we will refer to the outcome for $\newcasecharac$ as $\paacbr(D,\newcasecharac)$.\footnote{In the notation we omit $\defcase$, and leave it implicit instead for readability.} 
Unless otherwise stated, we will 
assume arbitrary $X$, $Y$, $D$, $\newcasearg$, and $\defcase$ (satisfying the  previously defined 
constraints).
Finally, we will refer to $\paacbr$ instantiated with $\pgeq=\supseteq$ and $\defcase = \defcaseset$ as $\oaacbr$.

\label{sec:properties}

\section{Non-monotonicity analysis of classifiers}
\label{sec:orge5b6025}

{In this section we provide a generic analysis of the non-monotonicity properties of data-driven classifiers, using $D$, $X$ and $Y$ to denote generic inputs and outputs of classifiers, admitting our casebases, characterisations and outcomes as special instances. Later in the paper, we will apply this analysis to \paacbr\ and our modification thereof. 
} 
Typically, a classifier can be understood as a function from an input set $X$ to an output set $Y$.  In machine learning, classifiers are obtained by {training with} an initial, finite $D \subseteq (X \times Y)$, called the training set.
{In (any form of) \aacbr, $D$ can also be seen as a training set of sorts.}  
Thus, we will characterise a classifier as a two-argument function \(\learn\) that maps from a dataset \(D{\subseteq (X\times Y)}\) and from a new input $x \in X$ to a prediction $y \in Y$.\footnote{{Notice that this understanding
relies upon the assumption that classifiers are deterministic. Of course this is not the case for many machine learning models, e.g. artificial neural networks trained using stochastic gradient descent and randomised hyperparameter search. This understanding is however in line with recent work using 
decision functions as approximations of classifiers whose output needs 
explaining (e.g. see \cite{Shih_19}). 
Moreover, it works well 
when analysing $\PAACBR$.}} 
{Notice that this function is total, in line with the common assumptions that classifiers generalise beyond their training dataset.}

Let us model directly the relationship between the dataset \(D\) and the predictions it makes {via the classifier} as an inference system in the following way:

\begin{definition} \label{def:non-monot-analysis}
	Given a classifier \(\learn{:\powerset{(X\times Y)}\times X \rightarrow Y}\), let $\lang=\lang^{+} \cup \lang^{-}$ be a language consisting of atoms \(\lang^{+} = X \times Y\) and negative sentences  
\(\lang^{-} \!= \!\{\neg(x,y) | $ $(x,y) \in X \times Y\}\). 
Then,  \(\infers_{\learn}\) is an \emph{inference relation} from \(2^{\lang^{+}}\) to \(\lang\) such that 
  \begin{itemize}
	  \item \(D \infers_{\learn} (x,y)\), iff  {\(\learn(D,x) = y\)};
  \item \(D \infers_{\learn} \neg(x,y)\), {iff there is a $y'$ such that}
	  {\(\learn(D,x) = y'\)}
          and \(y' \neq y\).\footnote{{We could equivalently have defined
              \(D \infers_{\learn} \neg(x,y)\) iff \(\learn(D,x) \neq y\).
              We have not done so as the used definition can be generalised for a scenario in which $\learn$ is not necessarily a total function. This scenario is left for future work.}}
  \end{itemize}
\end{definition}
{Intuitively, $\learn$ defines a language consisting of atoms (representing 
labelled examples) and their negations, and $\infers_{\learn}$ applies a sort of closed world assumption around $\learn$. 
}

Then, we can study  non-monotonicity properties (see Section \ref{sec:non-monot-prop})  of $\infers_{\learn}$. Here, completeness causes them to collapse.

\begin{theorem}\label{theo:compl}
\label{theo:cm-cut}
\label{theo:rm}
\begin{enumerate}
\item 
$\infers_{\learn}$ is \emph{complete}, i.e.		
  for every $(x,y)\in (X\times Y)$, either $D \infers_{\learn} (x,y)$ or $D \infers_{\learn} \neg (x,y)$.
\item {$\infers_{\learn}$ is \emph{consistent}, i.e.
    for every $(x,y)\in (X\times Y)$, it does not hold that both $D \infers_{\learn} (x,y)$ and $D \infers_{\learn} \neg (x,y)$.}
\item $\infers_{\learn}$ is cautiously monotonic iff it satisfies cut.
\item $\infers_{\learn}$ is cautiously monotonic {iff} it is cumulative.
\item $\infers_{\learn}$ is cautiously monotonic {iff} it satisfies rational monotonicity.
\end{enumerate}
\end{theorem}

\section{Cautious monotonicity in \texorpdfstring{$\bm{\pAACBR}$}{$\pAACBR$}}
\label{sec:orgef4abe6}

Our first main result is about (lack of) cautious monotonicity of 
{the inference relation drawn from the classifier $\paacbr(D,\newcasecharac)$. } 
\begin{theorem}
  \label{theo:aacbr-not-caut-mono}
	{$\vdash_{\paacbr}$} is not cautiously monotonic.
\end{theorem}

\begin{proof}
	{We show a counterexample, choosing $X = \powerset{\{a,b,c,z\}}$, $Y = \{-, +\}$, and $\pleq = \supseteq$. Define
	\(D = \{\caseset[a]{+},\) \(\caseset[c]{+},\) \(
          \caseset[a,b]{-},\) \( \caseset[c,z]{-}\}\) {and \defcase = \defcaseset\ }, and  two new cases: \(N_1 = \{a,b,c\}\) and \(N_2 = \{a,b,c,z\}\).
        }

Consider now {$\paacbr(D,N_1)$ and $\paacbr(D,N_2)$}. We can see in Figure \ref{fig:n_1} that \(D \vdash_{\paacbr} (N_1, +)\) and in Figure \ref{fig:n_2} that \(D \vdash_{\paacbr} (N_2, -)\).

Finally, let us consider {$\aaFtwo{D\cup\{\case{N_1}{+}\}}{N_2})$}
in Figure \ref{fig:n_2_alt}. We can then conclude that \(D \cup \{\case{N_1}{+}\} \vdash_{\paacbr} (N_2, +)\) even though \(D \vdash_{\paacbr} (N_1,+)\) and \(D \vdash_{\paacbr} (N_2, -)\), as required.
\end{proof}

\begin{figure*}[t!hb]
  \centering
  \begin{subfigure}[t]{0.30\textwidth}
    \centering
\begin{tikzpicture}[>=latex,line join=bevel, scale=0.8, font=\scriptsize]
  \pgfsetlinewidth{1bp}
\begin{scope}
  \pgfsetstrokecolor{black}
  \definecolor{strokecol}{rgb}{1.0,1.0,1.0};
  \pgfsetstrokecolor{strokecol}
  \definecolor{fillcol}{rgb}{1.0,1.0,1.0};
  \pgfsetfillcolor{fillcol}
  \filldraw (0.0bp,0.0bp) -- (0.0bp,127.0bp) -- (181.43bp,127.0bp) -- (181.43bp,0.0bp) -- cycle;
\end{scope}
  \pgfsetcolor{black}
\draw [->] (47.811bp,88.933bp) .. controls (53.559bp,92.423bp) and (60.279bp,96.503bp)  .. (75.443bp,105.71bp);
\draw [->] (30.587bp,58.66bp) .. controls (30.598bp,58.75bp) and (30.608bp,58.84bp)  .. (31.62bp,67.442bp);
\draw [->] (134.27bp,88.404bp) .. controls (127.63bp,92.165bp) and (119.65bp,96.688bp)  .. (103.35bp,105.93bp);
\draw [->] (151.61bp,58.66bp) .. controls (151.6bp,58.75bp) and (151.59bp,58.84bp)  .. (150.58bp,67.442bp);
\draw [->] (111.4bp,23.133bp) .. controls (116.34bp,25.844bp) and (121.7bp,28.783bp)  .. (135.87bp,36.552bp);
\begin{scope}
  \definecolor{strokecol}{rgb}{0.0,0.0,0.0};
  \pgfsetstrokecolor{strokecol}
  \definecolor{fillcol}{rgb}{0.83,0.83,0.83};
  \pgfsetfillcolor{fillcol}
  \filldraw [opacity=1] [thin] (29.1bp,46.0bp) ellipse (25.0bp and 12.5bp);
  \draw (29.098bp,46.0bp) node {$\caseset[a,b]{-}$};
\end{scope}
\begin{scope}
  \definecolor{strokecol}{rgb}{0.0,0.0,0.0};
  \pgfsetstrokecolor{strokecol}
  \draw [thin] (153.1bp,46.0bp) ellipse (25.0bp and 12.5bp);
  \draw (153.1bp,46.0bp) node {$\caseset[c,z]{-}$};
\end{scope}
\begin{scope}
  \definecolor{strokecol}{rgb}{0.0,0.0,0.0};
  \pgfsetstrokecolor{strokecol}
  \draw [thin] (89.1bp,114.0bp) ellipse (18.0bp and 12.5bp);
  \draw (89.098bp,114.0bp) node {$\defcaseset$};
\end{scope}
\begin{scope}
  \definecolor{strokecol}{rgb}{0.0,0.0,0.0};
  \pgfsetstrokecolor{strokecol}
  \draw [thin] (33.1bp,80.0bp) ellipse (20.5bp and 12.5bp);
  \draw (33.098bp,80.0bp) node {$\caseset[a]{+}$};
\end{scope}
\begin{scope}
  \definecolor{strokecol}{rgb}{0.0,0.0,0.0};
  \pgfsetstrokecolor{strokecol}
  \definecolor{fillcol}{rgb}{0.83,0.83,0.83};
  \pgfsetfillcolor{fillcol}
  \filldraw [opacity=1] [thin] (149.1bp,80.0bp) ellipse (20.0bp and 12.5bp);
  \draw (149.1bp,80.0bp) node {$\caseset[c]{+}$};
\end{scope}
\begin{scope}
  \definecolor{strokecol}{rgb}{0.0,0.0,0.0};
  \pgfsetstrokecolor{strokecol}
  \definecolor{fillcol}{rgb}{0.83,0.83,0.83};
  \pgfsetfillcolor{fillcol}
  \filldraw [thin] (127.1bp,12.0bp) -- (109.1bp,24.5bp) -- (73.1bp,24.5bp) -- (55.1bp,12.0bp) -- (73.1bp,-0.5bp) -- (109.1bp,-0.5bp) -- cycle;
  \draw (91.098bp,12.0bp) node {$\caseset[a,b,c]{?}$};
\end{scope}
\end{tikzpicture}
\caption{$\aaFtwo{D}{N_1}$}
\label{fig:n_1}
\end{subfigure}
\quad \begin{subfigure}[t]{0.30\textwidth}
  \centering
\begin{tikzpicture}[>=latex,line join=bevel, scale=0.8, font=\scriptsize]
  \pgfsetlinewidth{1bp}
\pgfsetcolor{black}
\draw [->] (62.166bp,88.756bp) .. controls (68.351bp,92.381bp) and (75.661bp,96.666bp)  .. (91.34bp,105.86bp);
\draw [->] (43.464bp,58.66bp) .. controls (43.48bp,58.75bp) and (43.496bp,58.84bp)  .. (45.014bp,67.442bp);
\draw [->] (150.4bp,88.404bp) .. controls (143.76bp,92.165bp) and (135.78bp,96.688bp)  .. (119.48bp,105.93bp);
\draw [->] (168.37bp,58.66bp) .. controls (168.35bp,58.75bp) and (168.34bp,58.84bp)  .. (167.08bp,67.442bp);
\begin{scope}
  \definecolor{strokecol}{rgb}{0.0,0.0,0.0};
  \pgfsetstrokecolor{strokecol}
  \definecolor{fillcol}{rgb}{0.83,0.83,0.83};
  \pgfsetfillcolor{fillcol}
  \filldraw [opacity=1] [thin] (105.23bp,114.0bp) ellipse (18.0bp and 12.5bp);
  \draw (105.23bp,114.0bp) node {$\defcaseset$};
\end{scope}
\begin{scope}
  \definecolor{strokecol}{rgb}{0.0,0.0,0.0};
  \pgfsetstrokecolor{strokecol}
  \draw [thin] (47.23bp,80.0bp) ellipse (20.5bp and 12.5bp);
  \draw (47.23bp,80.0bp) node {$\caseset[a]{+}$};
\end{scope}
\begin{scope}
  \definecolor{strokecol}{rgb}{0.0,0.0,0.0};
  \pgfsetstrokecolor{strokecol}
  \definecolor{fillcol}{rgb}{0.83,0.83,0.83};
  \pgfsetfillcolor{fillcol}
  \filldraw [opacity=1] [thin] (41.23bp,46.0bp) ellipse (25.0bp and 12.5bp);
  \draw (41.23bp,46.0bp) node {$\caseset[a,b]{-}$};
\end{scope}
\begin{scope}
  \definecolor{strokecol}{rgb}{0.0,0.0,0.0};
  \pgfsetstrokecolor{strokecol}
  \draw [thin] (165.23bp,80.0bp) ellipse (20.0bp and 12.5bp);
  \draw (165.23bp,80.0bp) node {$\caseset[c]{+}$};
\end{scope}
\begin{scope}
  \definecolor{strokecol}{rgb}{0.0,0.0,0.0};
  \pgfsetstrokecolor{strokecol}
  \definecolor{fillcol}{rgb}{0.83,0.83,0.83};
  \pgfsetfillcolor{fillcol}
  \filldraw [opacity=1] [thin] (170.23bp,46.0bp) ellipse (25.0bp and 12.5bp);
  \draw (170.23bp,46.0bp) node {$\caseset[c,z]{-}$};
\end{scope}
\begin{scope}
  \definecolor{strokecol}{rgb}{0.0,0.0,0.0};
  \pgfsetstrokecolor{strokecol}
  \definecolor{fillcol}{rgb}{0.83,0.83,0.83};
  \pgfsetfillcolor{fillcol}
  \filldraw [thin] (144.73bp,12.0bp) -- (124.98bp,24.5bp) -- (85.48bp,24.5bp) -- (65.73bp,12.0bp) -- (85.48bp,-0.5bp) -- (124.98bp,-0.5bp) -- cycle;
  \draw (105.23bp,12.0bp) node {$\caseset[a,b,c,z]{?}$};
\end{scope}
\end{tikzpicture}
\caption{$\aaFtwo{D}{N_2}$}
\label{fig:n_2}
\end{subfigure}
\quad \begin{subfigure}[t]{0.30\textwidth}
  \centering
\begin{tikzpicture}[>=latex,line join=bevel, scale=0.8, font=\scriptsize]
  \pgfsetlinewidth{1bp}
\pgfsetcolor{black}
\draw [->] (72.603bp,89.291bp) .. controls (77.364bp,92.465bp) and (82.807bp,96.094bp)  .. (96.45bp,105.19bp);
\draw [->] (56.156bp,58.66bp) .. controls (56.167bp,58.75bp) and (56.177bp,58.84bp)  .. (57.189bp,67.442bp);
\draw [->] (149.48bp,88.933bp) .. controls (144.15bp,92.29bp) and (137.95bp,96.194bp)  .. (123.2bp,105.48bp);
\draw [->] (166.18bp,58.66bp) .. controls (166.17bp,58.75bp) and (166.16bp,58.84bp)  .. (165.14bp,67.442bp);
\draw [->] (40.609bp,24.275bp) .. controls (40.868bp,24.674bp) and (41.129bp,25.077bp)  .. (46.957bp,34.085bp);
\begin{scope}
  \definecolor{strokecol}{rgb}{0.0,0.0,0.0};
  \pgfsetstrokecolor{strokecol}
  \draw [thin] (109.67bp,114.0bp) ellipse (18.0bp and 12.5bp);
  \draw (109.67bp,114.0bp) node {$\defcaseset$};
\end{scope}
\begin{scope}
  \definecolor{strokecol}{rgb}{0.0,0.0,0.0};
  \pgfsetstrokecolor{strokecol}
  \definecolor{fillcol}{rgb}{0.83,0.83,0.83};
  \pgfsetfillcolor{fillcol}
  \filldraw [opacity=1] [thin] (58.67bp,80.0bp) ellipse (20.5bp and 12.5bp);
  \draw (58.667bp,80.0bp) node {$\caseset[a]{+}$};
\end{scope}
\begin{scope}
  \definecolor{strokecol}{rgb}{0.0,0.0,0.0};
  \pgfsetstrokecolor{strokecol}
  \draw [thin] (54.67bp,46.0bp) ellipse (25.0bp and 12.5bp);
  \draw (54.667bp,46.0bp) node {$\caseset[a,b]{-}$};
\end{scope}
\begin{scope}
  \definecolor{strokecol}{rgb}{0.0,0.0,0.0};
  \pgfsetstrokecolor{strokecol}
  \draw [thin] (163.67bp,80.0bp) ellipse (20.0bp and 12.5bp);
  \draw (163.67bp,80.0bp) node {$\caseset[c]{+}$};
\end{scope}
\begin{scope}
  \definecolor{strokecol}{rgb}{0.0,0.0,0.0};
  \pgfsetstrokecolor{strokecol}
  \definecolor{fillcol}{rgb}{0.83,0.83,0.83};
  \pgfsetfillcolor{fillcol}
  \filldraw [opacity=1] [thin] (167.67bp,46.0bp) ellipse (25.0bp and 12.5bp);
  \draw (167.67bp,46.0bp) node {$\caseset[c,z]{-}$};
\end{scope}
\begin{scope}
  \definecolor{strokecol}{rgb}{0.0,0.0,0.0};
  \pgfsetstrokecolor{strokecol}
  \definecolor{fillcol}{rgb}{0.83,0.83,0.83};
  \pgfsetfillcolor{fillcol}
  \filldraw [opacity=1] [thin] (32.67bp,12.0bp) ellipse (29.0bp and 12.5bp);
  \draw (32.667bp,12.0bp) node {$\caseset[a,b,c]{+}$};
\end{scope}
\begin{scope}
  \definecolor{strokecol}{rgb}{0.0,0.0,0.0};
  \pgfsetstrokecolor{strokecol}
  \definecolor{fillcol}{rgb}{0.83,0.83,0.83};
  \pgfsetfillcolor{fillcol}
  \filldraw [thin] (207.17bp,12.0bp) -- (187.42bp,24.5bp) -- (147.92bp,24.5bp) -- (128.17bp,12.0bp) -- (147.92bp,-0.5bp) -- (187.42bp,-0.5bp) -- cycle;
  \draw (167.67bp,12.0bp) node {$\caseset[a,b,c,z]{?}$};
\end{scope}
\end{tikzpicture}
\caption{$\aaFtwo{D\cup \{\case{N_1}{+}\}}{N_2}$}
\label{fig:n_2_alt}
\end{subfigure}
\caption{AFs for the proof of Theorem~\ref{theo:aacbr-not-caut-mono}, with the grounded extension shaded.}
\end{figure*}
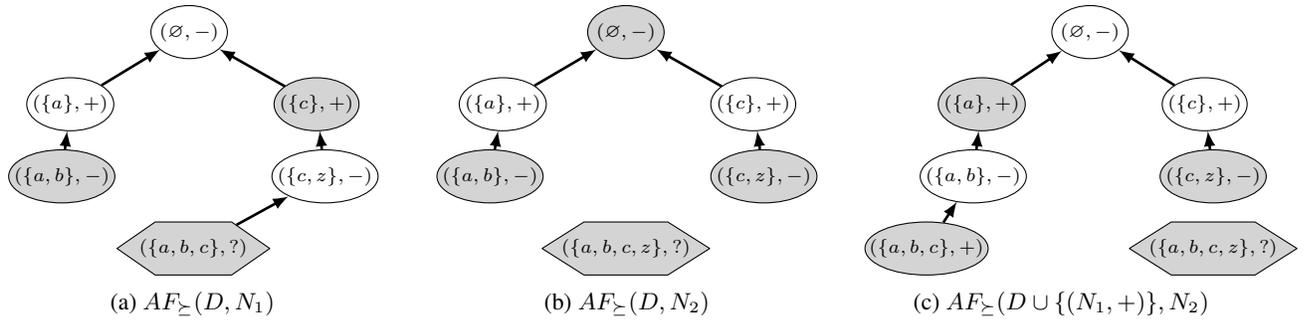

Note that the proof of Theorem~\ref{theo:aacbr-not-caut-mono} shows that the inference relation drawn from the original \aacbr\ (i.e. \oaacbr) is also non-cautiously monotonic, given that the proof's counterexample is obtained with \oaacbr. Also, note that the counterexample amounts to an expansion of Example~\ref{running-example}, as follows.

\begin{example} (Example~\ref{running-example} cont.)
  \label{running-example-2} 
Assume that a different type of crime happened: public offending someone's honour, which we will call defamation (\emph{df}). In one case, it was established that the defendant did publicly damage someone's honour, and was considered guilty \(\case{\{df\}}{+}\). In a subsequent case, even if proven that the defendant did hurt someone's honour, it was established that this was done by a true allegation, and thus the case was dismissed, represented as \(\case{\{df,td\}}{-}\).
What happens, then, if a same defendant is:
1. simultaneously proven guilty of homicide, of defamation, but shown to have committed the homicide in self-defence ($\caseset[hm,df,sd]{?}$); or
2. simultaneously proven guilty of homicide, of defamation, shown to have committed the homicide in self-defence, and  also shown to have committed defamation by a true allegation ($\caseset[hm,df,sd,td]{?}$)?

  {We can map these situations to our counterexample in the proof of Theorem~\ref{theo:aacbr-not-caut-mono} by setting $a = hm$, $b = sd$, $c = df$, and $z = td$.
  The first question is answered by the AF represented in Figure \ref{fig:n_1}, with outcome $+$, that is, the defendant is considered guilty.}
The proof of Theorem~\ref{theo:aacbr-not-caut-mono} shows that the answer to the second question in $\paacbr$ would depend on whether the case in the first question was already judged or not. If not, then the cases $\caseset[hm,sd]{-}$ and $\caseset[df,td]{-}$ would be the nearest cases, and the outcome would be $-$, that is, not guilty. However, if the case in the first question was already judged and incorporated into the case law, it would serve as a counterargument for $\caseset[hm,sd]{-}$, and guarantee that the outcome is $+$ (guilty). Intuitively, this seems strange, in particular, because the order in which the case in the first answer is judged affects the case in the second question.
\end{example}

{This example aims only to illustrate an interpretation in which $\paacbr$ operates seemingly inappropriately. Whether this behaviour of $\oaacbr$ in particular is desirable depends on other elements such as, for $\paacbr$, the interrelation between the characterisations and the partial order.}

\section{A cumulative \texorpdfstring{$\bm{\pAACBR}$}{$\pAACBR$}}
\label{sec:org432032f}

We will now present {\caacbr, a novel, \emph{cumulative} incarnation of \aacbr} 
satisfying cautious monotonicity.

\paragraph{{\bf Concise dataset.}} 
Firstly, {let us present some general notions, defined in terms of the $\cinfers$ inference relation} from an arbitrary classifier $\learn$.
Intuitively, we {are after} a relation \(\cinfers'\) such that if \(D \cinfers c\) and \(D \cinfers d\), then \(D\cup\{c\} \cinfers' d\) {(in our concrete setting, $\cinfers=\infers_{\paacbr}$ and $\cinfers'=\infers_{\caacbr}$)}. We also want the property that, whenever \(D\) is ``well-behaved'' (in a sense to be made precise later), \(D \cinfers {s}\) iff \(D \cinfers' {s}\). In this way, given that \(D \cinfers' c\) and \(D \cinfers' d\), then we would conclude \(D\cup\{c\} \cinfers' d\), making \(\cinfers'\) a cautiously monotonic relation.
We will {define $\cinfers'$} by building a subset of the original dataset 
in such a way that cautious monotonicity is preserved. We start with the following {notion of \emph{\includable} {examples}}:
  
\begin{definition} \label{def:includable}
An example $(x,y) \in X \times Y$ is \emph{\surprising} \wrt\ $D$ iff $D\setminus\{(x,y)\} \not \cinfers (x,y)$ and \emph{sufficient} {\wrt} $D$ iff $D \cup \{(x,y)\} \cinfers (x,y)$. Additionally, an example $(x,y) \in X \times Y$ is \emph{\includable} \wrt\ $D$ iff it is both {\surprising} and {\sufficient} {\wrt} to $D$.
\end{definition}
The definition of {\includable} example has two parts: that the example is {\surprising}, in the sense that, without it, the predicted outcome would be different, and that it is {\sufficient}, in the sense that adding it makes it inferable.  We then define the notion of \emph{concise} subsets, {amounting to {\includable} examples only:} 

\begin{definition} \label{def:concise}
	Let $S \subseteq X \times Y$ be a dataset, 
$S' \subseteq S$, and let $\phi(S') = \{(x,y) \in S \mid (x,y) \text{ is {\includable} \wrt\ } S' \}$.  Then $S'$ is \emph{concise \wrt\ $S$}  whenever it is a fixed point of $\phi$, that is, $\phi(S') = S'$. \end{definition}
To illustrate {in the context of \aacbr}, consider {$S$ from which the AF} in Figure \ref{fig:n_2_alt} {is drawn}. $S$ is not concise {\wrt\ itself}, since $\caseset[a,b,c]{+}$ is not {\includable} \wrt\ $S$ {(indeed, $S\setminus \{\caseset[a,b,c]{+}\} \infers_{\paacbr} (\{a,b,c\},+)$, see Figure~\ref{fig:n_1})}. Also, $S' = S \setminus \{\caseset[a,b]{-}, \caseset[a,b,c]{+}\}$ {is not concise either (\wrt\ $S$)}, as $\caseset[a,b]{-}$ is {\includable} \wrt\ $S'$ (the predicted outcome being $+$), but not an element of $S'$. The only concise subset of $S$ here is $S'' = S \setminus \{\caseset[a,b,c]{+}\}$.

Let us now consider $D'\subseteq D$, for a dataset $D$. If \(D'\) is concise \wrt\ \(D\), \((x,y) \in (X\times Y) \setminus D\) is an example not in $D$ already and \(D' \cinfers (x,y) \), then \((x,y)\) is not {\includable} \wrt\ \(D'\), and thus \(D'\) is still concise \wrt\ \(D \cup \{(x,y)\}\). 
Now, \emph{suppose that there is exactly one such concise} $D'\subseteq D$ {\wrt\ $D$}  {(let us refer to this subset simply as $concise(D)$).}
Then, it seems attractive to define \(\cinfers'\) as: \({D} \cinfers' (x,y)\) iff \(concise({D}) \cinfers (x,y)\). Such \(\cinfers'\) inference relation would then be cautiously monotonic if \(concise({D}) = concise({D} \cup \{(x,y)\})\). To see that, consider \((x',y') \in (X\times Y)\) such that \(D \cinfers' (x',y')\). Then, since \(concise({D}) = concise({D} \cup \{(x,y)\})\), for our new $\cinfers'$, \(D\) and \(D \cup \{(x,y)\}\) would infer the exact same sentences, thus \(D \cup \{(x,y)\} \cinfers' (x',y')\).
This equality is indeed guaranteed given that $(x,y) \not \in D$, thus it is not {\includable}, and then a concise subset of \({D}\) is still a concise subset of \({D} \cup \{(x,y)\}\) (otherwise, if \((x,y) \in D\), the equality would be trivial).
Note that concision is too strong a property here: all that is needed is that a subset $D'$ is selected such that every case in it is surprising \wrt\ $D'$ itself. However, concision implies that as many cases are added as possible, while restricting to the ones that guarantee their outcomes.

In the remainder, we state uniqueness and 
give an algorithm that constructs $concise(D)$, in the case of $\PAACBR$. If $D$ is in{\coherent}, there might be no concise subset thereof, but our method will still be useful, as we discuss later.

\paragraph{{\bf Uniqueness and algorithm.}} We first give a property of concise subsets:
\begin{theorem} \label{theo:unique-concise}
  For $\paacbr$, if there is
	a concise $D' \! \subseteq \! D$ \wrt\ $D$ then every concise subset of $D$ \wrt\ $D$ is the same as $D'$. 
	
\end{theorem}

\SetAlgoSkip{}
\DecMargin{0.4em}
\SetInd{0.5em}{0.6em}
\begin{algorithm}[t]
  \KwIn{An $\PAACBR$ framework $\AF{\Args}{\attacks}$ and a case $n = \fullcase{n}$}
  \KwOut{A new $\PAACBR$ $\AF{\Args'}{\attacks'}$ framework}

  $DEF \leftarrow \{(x,y) \in \aaFtwo{Args}{ \charac{n}} \mid $ \\
  \hspace{5em} $(x,y) \neq \newcasearg[n]$ and $\newcasearg[n]$ defends \\
  \hspace{5em} $(x,y)$ in $\aaFtwo{Args}{ \charac{n}} \}$\;

  $\Args' \leftarrow \Args \cup \{n\}$\;
  $\attacks' \leftarrow (\attacks \cup \{ (n,a) \mid a = \fullcase{a}, a \in DEF, $\\
  \hspace{8em}\text{ and } $a_o \neq n_o \})$\;
\KwRet{$\AF{\Args'}{\attacks'}$}
  
\caption{$simple\_add$ for $\PAACBR$.}
  \label{algo:simpleadd}
\end{algorithm}
\begin{algorithm}[t]
  \KwIn{A dataset $D$}
  \KwOut{A subset $D'$ of $D$, an AF $\cAACBR(D)$}

  $unprocessed \leftarrow$ $D$\;
  $Args \leftarrow \{\defcase\}$\;
  $\attacks \leftarrow \emptyset$\;

  \While{$unprocessed \neq \emptyset$}{
$stratum \leftarrow \{(x,y) \in unprocessed \mid$ \\
    \hspace{4.0em}$(x,y)$\text{ is } $\pleq\!\text{-minimal}$\text{ in }$ unprocessed\}$\;
    $unprocessed \leftarrow unprocessed \setminus stratum$\;
    $to\_add \leftarrow \emptyset$\;
    \For{$case \in stratum$}{
      $(case\_charac, case\_outcome) \leftarrow case$\;
      \If{the outcome for $case\_charac$ \wrt\ $(\Args, \attacks)$\\
        \hspace{1em}is not $case\_outcome$}{
        $to\_add \leftarrow to\_add \cup \{case\}$\;}}
      
    \For{$case \in to\_add$}
    {${(Args,\attacks)}\leftarrow$$simple\_add({(Args,\attacks)},case)$\;}}
  $D' \leftarrow Args \setminus \{\defcase\}$\;
  \KwRet{$D', (Args,\attacks)$}

\caption{Setup/learning for $\cAACBR$.}
  \label{algo:cumulaacbr}
\end{algorithm}
{The procedure for finding this unique $concise(D)$, if it exists, is integrated within Algorithm \ref{algo:cumulaacbr}, using in turn Algorithm~\ref{algo:simpleadd}.
If $concise(D)$ does not exist, the algorithm will still return some $D' \subseteq D$ consisting only of  {\surprising} examples {\wrt} $D'$.
\gppnew{In fact, Algorithm \ref{algo:cumulaacbr} returns both this subset and its corresponding AF, in order to make implementation more straighforward.}
The main idea {behind the algorithm is} simple: we start with the default argument, and progressively build the AF by adding cases from {$D$} by following the partial order $\pleq$. Before adding a past case, we test whether it is {\includable} or not \wrt\ {the dataset underpinning the current} AA framework: if it is, then it is added; otherwise, it is not.
{More precisely, the algorithm works with strata over $D$, alongside $\pleq$.}
{In the simplest setting where each stratum is a singleton, the algorithm works as follows:} starting with \(D_0 = \emptyset\) and the entire dataset $D = \{d_i\}_{i \in \{1, \dots, \card{D}\}}$ unprocessed, at each step $i$, we {obtain} either \(D_{i} = D_{i-1} \cup \{d_{i}\}\), if \(d_{i}\) is {\includable} \wrt\ \(D_{i-1}\), and \(D_{i} = D_{i-1}\), otherwise. 
Then \(\hat{D} = D_{\card{D}} \subseteq D\) is the (implicit) result of the algorithm. 
Each example {of the current stratum} is tested for ``{\includability}'' with respect to the same (current) subset $D_i$, and only the {\includable} examples are added to it. Here, however, testing for {\surprise} is enough for this verification.  
We illustrate the application of the algorithms next.

\begin{example}

  Once more consider the dataset \(D = \{\caseset[a]{+}, \caseset[c]{+}, \caseset[a,b]{+},\) \(\caseset[c,z]{+}, \caseset[a,b,c]{+}\}\) in Figure \ref{fig:n_2_alt}, as well as the definitions used in the proof of Theorem~\ref{theo:aacbr-not-caut-mono} for $X$, $Y$, $\defcase$ and $\pleq$. Let us examine the application of Algorithm \ref{algo:cumulaacbr} to it. We start with an AA framework $AF_0$ consisting only of $\defcase$, that is, $D_0 = \emptyset$, $AF_0 = \aaFone{D_0} = \aaFone{\emptyset} = \AF{\{\defcaseset\}}{\emptyset}$. The first stratum would consist of $stratum_1 = \{\caseset[a]{+},\caseset[c]{+}\}$. Of course, then, we have $\paacbr(\{\defcaseset\},$ $\set{a}) = -$, and similarly for $\caseset[c]{?}$. Thus, every argument in $stratum_1$ is {\includable}, and are then included in the next $AF$, resulting in $D_1 = \caseset[a]{+},\caseset[c]{+}$ and $AF_1 = \aaFone{D_1}$.
{Now, the second stratum is $stratum_2 = \{\caseset[a,b]{-},\caseset[c,z]{-}\}$. We can verify that $\paacbr(D_1, \set{a,b}) = +$ and $\paacbr(D_1, \set{c,z}) = +$. As a result $\caseset[a,b]{-}$ and $\caseset[c,z]{-}$ are both {\includable}, and then included in next step, that is, $D_2 = D_1 \allowbreak\cup\allowbreak \{\caseset[a,b]{-},\allowbreak\caseset[c,z]{-}\}$, and $AF_2 = \aaFone{D_2}$.}
Finally, $stratum_3 = \{\caseset[a,b,c]{+}\}$. Now we verify that $\caseset[a,b,c]{+}$ is \emph{not} {\includable}, because \(\paacbr(D_2,\set{a,b,c}) = +\). Therefore it is \emph{not} added to the AA framework, that is, $D_3 = D_2$ and thus $AF_3 = \aaFone{D_3} = \aaFone{D_2} = AF_2$. Now $unprocessed = \emptyset$, and the selected subset is $D_3$, with corresponding $\aaFone{D_3} = AF_3$, and we are done. 
	Note	that using $\caacbr$ the counterexample in the proof of Theorem~\ref{theo:aacbr-not-caut-mono} would fail, since $\caseset[a,b,c]{+}$ would not have been added to the AF.
\end{example}

Note that\ifincludeincoherent, if $D$ is \coherent, \fi we could have defined the algorithm equivalently by looking at cases one-by-one rather than grouping them in strata.
However, using strata still has the advantage of allowing for parallel testing of new cases. 
\ifincludeincoherent
If $D$ is in\coherent, then using strata is necessary.
\fi

A full complexity analysis of the algorithm is outside the scope of this paper. However, note here that the algorithm refrains from building the AA framework from 
scratch each time a new case is considered.
Still regarding Algorithm~\ref{algo:simpleadd}, note that it is
easy to compute the set DEF while checking whether the next case
is {\includable} or not,  thus we could optimise its implementation
with the use of caching. Besides, the subset of minimal cases (that is, the stratum)
can be extracted efficiently by representing the partial order
as a directed acyclic graph and traversing this graph.
Finally, the order in which the cases in the same stratum are added does not affect the outcome.
Thus, each case in the same stratum can be safely tested for {\includability} in parallel.

\paragraph{\texorpdfstring{$\bm{\cAACBR}$.}{$\cAACBR$.}}
\begin{theorem} \label{theo:existence-concise}
Let $D'$ be the dataset \gppnew{returned by} Algorithm \ref{algo:cumulaacbr}. Then for every \(\casei \in D'\), \(\casei\) is surprising {\wrt} $D'$. Additionally, if $D$ has a concise subset, $D'$ is its unique concise subset. In particular, there is always a concise subset if $D$ is {\coherent}. 
\end{theorem}
We cannot generalise the existence result for any $D$: consider the (in{\coherent}) counterexample when $D = \{\caseset[a]{+},$ $ \caseset[a,b]{-}, \caseset[a,b]{+}\}$, for $\oaacbr$. None of its subsets is concise. Still, our algorithm returns the subset $\{\caseset[a]{+},$ $ \caseset[a,b]{-}\}$, which is coherent and consisting only of surprising examples.

To conclude, we can then define inference in $\caacbr$, the classifier yielded by the strategy described until now:
\begin{definition}
  Let $D$ be a dataset and $D'$ be the subset of $D$ identified by Algorithm \ref{algo:cumulaacbr}. Let $\caaDN$ be the AF mined from $D'$ and $\newcasearg$, with default argument $\defcase$.  Then, $\caacbr(D,\newcasecharac)$ stands for the outcome for $\newcasecharac$, given $\caaDN$. 
\end{definition}
Thus, we directly obtain the inference relation $\infers_{\caacbr}$.
Then, \caacbr\ amounts to the form of \aacbr\ using this inference relation. It is easy to see, in line with the discussion before Theorem~\ref{theo:unique-concise} and using Theorem~\ref{theo:cm-cut}, that \caacbr\ satisfies several non-monotonicity properties, as follows:

\begin{theorem}
  $\infers_{\caacbr}$ is cautiously monotonic and also satisfies cut, cumulativity, and rational monotonicity.
\end{theorem}

\ifincludeincoherent
\paragraph{{\bf In{\coherence}.}}
\label{sec:org1365df6}
An important additional property of $\cAACBR$ is that it naturally accommodates a way to handle in{\coherence}s in the dataset. During the execution of Algorithm \ref{algo:cumulaacbr}, an in{\coherent} pair of cases would be considered at the same stratum. As every characterisation receives an outcome in a $\cAACBR$ framework, and exactly one, then if there is an in{\coherent} pair in the dataset, one of its examples would be {\includable} while the other would not. Therefore only the {\includable} example becomes an argument in the AA framework. Although an in{\coherent} dataset may not have a concise subset, this approach finds a {\coherent} subset which always chooses among one of the conflicting examples, using {\includability} as the criterion for choice.\footnote{Indeed, from this reasoning one can also see that \emph{every concise subset is also {\coherent}}.}
As an example, consider again Figure \ref{fig:example-legal-3}. Following Algorithm \ref{algo:cumulaacbr}, we see that in the first \texttt{while} loop both $\caseset[hm]{+}$ and $\caseset[hm]{-}$ are in the stratum. Since the default outcome is $-$, $\caseset[hm]{+}$ is a surprising case \wrt{} $\emptyset$ and thus is added, while $\caseset[hm]{-}$ is not and thus is not added, and the algorithm terminates.
\begin{theorem} \label{theo:caacbr-coherent}
  The dataset \gppnew{returned by} Algorithm \ref{algo:cumulaacbr} is \coherent.
\end{theorem}
Note that now that a {\coherent} subset is used as basis for the inference, whenever the default case is not in the grounded extension, it will be attacked by a case which is in it.\footnote{In more detail, this is so since the AF would be well-founded, and thus every argument outside the grounded extension would be attacked by it (see \citealp{Dung:95}).} Thus we have a ``principled'' way of dealing with in{\coherence}s, in which the {\includable} example is always kept.

\fi

\paragraph{Spikes.}
An inconvenience in $\paacbr$ is the presence of cases in the AF which do not reach the default case.  While part of the AF, they do not affect whether the default case $\defcase$ is or not in the grounded extension, and thus the outcome.
Formally, these cases can be defined as follows, for  $\paacbr$ as well as $\caacbr$:

\begin{definition}
  Let \(\myAF{} = \aaDN\) or \({\myAF{} = \caaDN{}}\), and \(\casei \in D \cap \Args\). Then, $\arga$ is a \emph{spike} iff there is no path in \(\myAF\) from \arga{} to \(\defcase\).
\end{definition}

As a simple example, consider the casebase in Figure \ref{fig:n_2_alt}, and add $\caseset[b]{-}$ to it. It would be attacked by $\caseset[a,b,c]{+}$, but it would attack no other argument. Thus, $\caseset[b]{-}$ would not reach any other argument and would, then, be a spike.

Spikes are unhelpful, since their presence is entirely superfluous, that is, they can be removed with no change in outcome, for any new case. 
\begin{theorem} \label{theo:remove-spike}
  Let \(\myAF{} = \aaDN\) and \(\casei \in D \cap \Args\) be a spike. Then \(\paacbr(D \setminus \{\casei\}, \newcasecharac) = \paacbr(D, \newcasecharac)\).
\end{theorem}
Thus, a useful step in practice is removing spikes from the AF when visualising or storing (e.g. for caching), since the AF may become significantly leaner (indeed, we do this in the case study in Section \ref{sec:case-study}).

Instead, \caacbr{} shows no spikes, by construction, given that spikes are not \includable, and thus are not added to \caaDN. 
\begin{theorem} \label{theo:no-spikes}
  Let \(\myAF{} = \caaDN\). Then, there are no spikes in \(\Args\).
\end{theorem}

\section{Case study}
\label{sec:case-study}
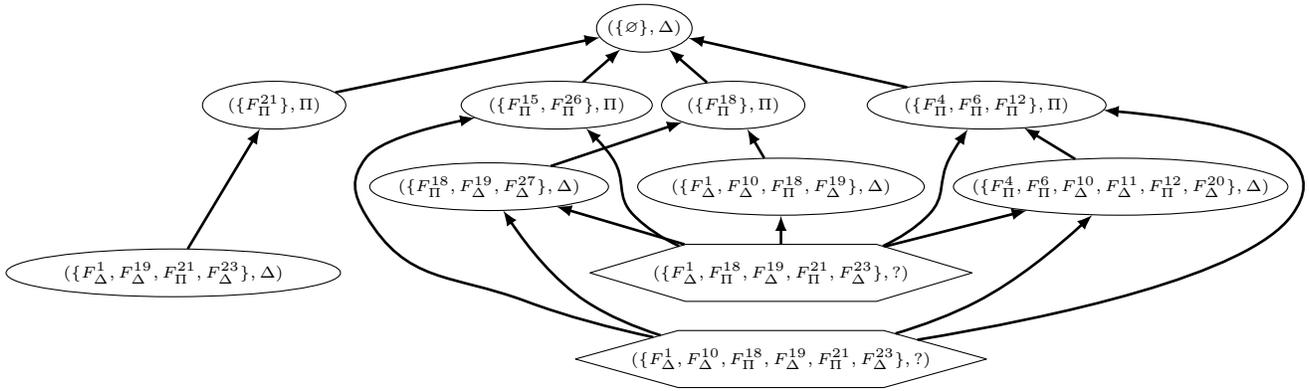
\begin{figure*}[h]
\centering
  \begin{tikzpicture}[>=latex,line join=bevel, scale=0.5, font=\tiny]
    \pgfsetlinewidth{1bp}
\begin{scope}
      \pgfsetstrokecolor{black}
      \definecolor{strokecol}{rgb}{1.0,1.0,1.0};
      \pgfsetstrokecolor{strokecol}
      \definecolor{fillcol}{rgb}{1.0,1.0,1.0};
      \pgfsetfillcolor{fillcol}
      \filldraw (0.0bp,0.0bp) -- (0.0bp,289.0bp) -- (977.9bp,289.0bp) -- (977.9bp,0.0bp) -- cycle;
    \end{scope}
    \begin{scope}
      \pgfsetstrokecolor{black}
      \definecolor{strokecol}{rgb}{1.0,1.0,1.0};
      \pgfsetstrokecolor{strokecol}
      \definecolor{fillcol}{rgb}{1.0,1.0,1.0};
      \pgfsetfillcolor{fillcol}
      \filldraw (0.0bp,0.0bp) -- (0.0bp,289.0bp) -- (977.9bp,289.0bp) -- (977.9bp,0.0bp) -- cycle;
    \end{scope}
    \pgfsetcolor{black}
\draw [->] (434.8bp,230.4bp) .. controls (441.24bp,236.06bp) and (448.49bp,242.43bp)  .. (463.01bp,255.19bp);
\draw [->] (136.89bp,104.62bp) .. controls (149.77bp,126.07bp) and (171.33bp,161.95bp)  .. (191.38bp,195.32bp);
\draw [->] (248.07bp,222.58bp) .. controls (300.74bp,233.53bp) and (386.3bp,251.31bp)  .. (447.58bp,264.05bp);
\draw [->] (571.41bp,173.01bp) .. controls (568.9bp,177.29bp) and (566.25bp,181.83bp)  .. (558.5bp,195.07bp);
\draw [->] (528.26bp,230.09bp) .. controls (521.71bp,235.76bp) and (514.32bp,242.16bp)  .. (499.45bp,255.03bp);
\draw [->] (410.43bp,167.02bp) .. controls (437.74bp,176.15bp) and (472.21bp,187.67bp)  .. (509.65bp,200.18bp);
\draw [->] (805.81bp,172.32bp) .. controls (796.17bp,178.19bp) and (785.64bp,184.6bp)  .. (767.12bp,195.88bp);
\draw [->] (679.16bp,226.45bp) .. controls (631.72bp,237.12bp) and (566.55bp,251.77bp)  .. (514.1bp,263.56bp);
\draw [->] (507.15bp,106.72bp) .. controls (494.59bp,112.65bp) and (482.58bp,120.26bp)  .. (473.0bp,130.0bp) .. controls (458.71bp,144.54bp) and (468.23bp,155.98bp)  .. (457.0bp,173.0bp) .. controls (453.43bp,178.41bp) and (449.01bp,183.62bp)  .. (437.04bp,195.51bp);
\draw [->] (584.0bp,108.14bp) .. controls (584.0bp,111.83bp) and (584.0bp,115.73bp)  .. (584.0bp,129.73bp);
\draw [->] (511.35bp,107.96bp) .. controls (483.15bp,116.3bp) and (451.23bp,125.73bp)  .. (414.66bp,136.53bp);
\draw [->] (661.85bp,106.27bp) .. controls (692.78bp,114.12bp) and (728.35bp,123.15bp)  .. (769.49bp,133.6bp);
\draw [->] (660.85bp,106.72bp) .. controls (673.41bp,112.65bp) and (685.42bp,120.26bp)  .. (695.0bp,130.0bp) .. controls (709.29bp,144.54bp) and (701.52bp,154.95bp)  .. (711.0bp,173.0bp) .. controls (713.34bp,177.46bp) and (716.1bp,182.0bp)  .. (724.87bp,194.9bp);
\draw [->] (487.78bp,37.979bp) .. controls (456.69bp,44.704bp) and (422.41bp,53.626bp)  .. (392.0bp,65.0bp) .. controls (334.82bp,86.384bp) and (302.33bp,77.606bp)  .. (271.0bp,130.0bp) .. controls (261.19bp,146.4bp) and (259.44bp,157.79bp)  .. (271.0bp,173.0bp) .. controls (280.35bp,185.3bp) and (312.17bp,194.82bp)  .. (353.67bp,203.53bp);
\draw [->] (493.29bp,39.365bp) .. controls (473.77bp,45.607bp) and (453.93bp,53.936bp)  .. (437.0bp,65.0bp) .. controls (413.32bp,80.473bp) and (393.19bp,105.9bp)  .. (374.6bp,133.56bp);
\draw [->] (670.35bp,40.749bp) .. controls (691.15bp,47.033bp) and (712.88bp,55.033bp)  .. (732.0bp,65.0bp) .. controls (761.47bp,80.362bp) and (790.62bp,104.38bp)  .. (818.65bp,130.15bp);
\draw [->] (686.79bp,36.022bp) .. controls (789.15bp,52.876bp) and (935.72bp,84.33bp)  .. (969.0bp,130.0bp) .. controls (980.26bp,145.45bp) and (981.3bp,158.37bp)  .. (969.0bp,173.0bp) .. controls (951.84bp,193.4bp) and (891.57bp,203.4bp)  .. (826.93bp,209.15bp);
\begin{scope}
      \definecolor{strokecol}{rgb}{0.0,0.0,0.0};
      \pgfsetstrokecolor{strokecol}
      \draw [thin] (415.0bp,213.0bp) ellipse (72.0bp and 18.0bp);
      \draw (415.0bp,213.0bp) node {$\caseset[F^{15}_\plaintiff, F^{26}_\plaintiff]{\plaintiff}$};
    \end{scope}
\begin{scope}
      \definecolor{strokecol}{rgb}{0.0,0.0,0.0};
      \pgfsetstrokecolor{strokecol}
      \draw [thin] (202.0bp,213.0bp) ellipse (54.0bp and 18.0bp);
      \draw (202.0bp,213.0bp) node {$\caseset[F^{21}_\plaintiff]{\plaintiff}$};
    \end{scope}
\begin{scope}
      \definecolor{strokecol}{rgb}{0.0,0.0,0.0};
      \pgfsetstrokecolor{strokecol}
      \draw [thin] (548.0bp,213.0bp) ellipse (54.0bp and 18.0bp);
      \draw (548.0bp,213.0bp) node {$\caseset[F^{18}_\plaintiff]{\plaintiff}$};
    \end{scope}
\begin{scope}
      \definecolor{strokecol}{rgb}{0.0,0.0,0.0};
      \pgfsetstrokecolor{strokecol}
      \draw [thin] (126.0bp,86.5bp) ellipse (126.0bp and 18.0bp);
      \draw (126.0bp,86.5bp) node {$\caseset[F^{1}_\defendant, F^{19}_\defendant, F^{21}_\plaintiff, F^{23}_\defendant]{\defendant}$};
    \end{scope}
\begin{scope}
      \definecolor{strokecol}{rgb}{0.0,0.0,0.0};
      \pgfsetstrokecolor{strokecol}
      \draw [thin] (728.0bp,86.5bp) -- (656.0bp,108.0bp) -- (512.0bp,108.0bp) -- (440.0bp,86.5bp) -- (512.0bp,65.0bp) -- (656.0bp,65.0bp) -- cycle;
      \draw (584.0bp,86.5bp) node {$\caseset[F^{1}_\defendant, F^{18}_\plaintiff, F^{19}_\defendant, F^{21}_\plaintiff, F^{23}_\defendant]{?}$};
    \end{scope}
\begin{scope}
      \definecolor{strokecol}{rgb}{0.0,0.0,0.0};
      \pgfsetstrokecolor{strokecol}
      \draw [thin] (481.0bp,271.0bp) ellipse (36.0bp and 18.0bp);
      \draw (481.0bp,271.0bp) node {$\caseset[\emptyset]{\defendant}$};
    \end{scope}
\begin{scope}
      \definecolor{strokecol}{rgb}{0.0,0.0,0.0};
      \pgfsetstrokecolor{strokecol}
      \draw [thin] (584.0bp,151.5bp) ellipse (108.0bp and 21.5bp);
      \draw (584.0bp,151.5bp) node {$\caseset[F^{1}_\defendant, F^{10}_\defendant, F^{18}_\plaintiff, F^{19}_\defendant]{\defendant}$};
    \end{scope}
\begin{scope}
      \definecolor{strokecol}{rgb}{0.0,0.0,0.0};
      \pgfsetstrokecolor{strokecol}
      \draw [thin] (364.0bp,151.5bp) ellipse (90.0bp and 18.0bp);
      \draw (364.0bp,151.5bp) node {$\caseset[F^{18}_\plaintiff, F^{19}_\defendant, F^{27}_\defendant]{\defendant}$};
    \end{scope}
\begin{scope}
      \definecolor{strokecol}{rgb}{0.0,0.0,0.0};
      \pgfsetstrokecolor{strokecol}
      \draw [thin] (840.0bp,151.5bp) ellipse (126.0bp and 21.5bp);
      \draw (840.0bp,151.5bp) node {$\caseset[F^{4}_\plaintiff, F^{6}_\plaintiff, F^{10}_\defendant, F^{11}_\defendant, F^{12}_\plaintiff, F^{20}_\defendant]{\defendant}$};
    \end{scope}
\begin{scope}
      \definecolor{strokecol}{rgb}{0.0,0.0,0.0};
      \pgfsetstrokecolor{strokecol}
      \draw [thin] (739.0bp,213.0bp) ellipse (90.0bp and 18.0bp);
      \draw (739.0bp,213.0bp) node {$\caseset[F^{4}_\plaintiff, F^{6}_\plaintiff, F^{12}_\plaintiff]{\plaintiff}$};
    \end{scope}
\begin{scope}
      \definecolor{strokecol}{rgb}{0.0,0.0,0.0};
      \pgfsetstrokecolor{strokecol}
      \draw [thin] (739.0bp,21.5bp) -- (661.5bp,43.0bp) -- (506.5bp,43.0bp) -- (429.0bp,21.5bp) -- (506.5bp,0.0bp) -- (661.5bp,0.0bp) -- cycle;
      \draw (584.0bp,21.5bp) node {$\caseset[F^{1}_\defendant, F^{10}_\defendant, F^{18}_\plaintiff, F^{19}_\defendant, F^{21}_\plaintiff, F^{23}_\defendant]{?}$};
    \end{scope}
\end{tikzpicture}
\caption{Resulting AA framework for the U.S. Trade Secrets casebase. \gppnew{Each case in the original dataset yields possibly many arguments, each argument represented by its factors and outcomes. Some of the factores are: $F^{1}_\defendant$: the plaintiff disclosed its product information in negotiations with defendant; $F^{21}_\plaintiff$: defendant obtained plaintiff's information altough he knew that plaintiff's information was confidential $F^{18}_\plaintiff$: defendant's product was identical to plaintiff's.}}
  \label{fig:case-study}
\end{figure*}

We now explore, as a case study for our approach, the US Trade Secrets domain, frequently discussed in the AI and Law literature \cite{DBLP:conf/icail/RisslandA87,DBLP:conf/icail/BruninghausA03,DBLP:journals/ail/Bench-Capon17}.
This area of law deals with misappropriation of commercially relevant information that, allegedly, should not have been available or used by another party. The stereotypical scenario is of a company, the plaintiff, suing another, the defendant, claiming that such misappropriation happened, resulting in economic loss for the plaintiff.
In this setting, each case is represented by \emph{factors} each supporting either plaintiff or defendant, and an outcome, which may be a win for plaintiff ($\plaintiff$) or defendant ($\defendant$). 
Formally, each such case is of the form $(F^\plaintiff,F^\defendant,o)$ where $F^\plaintiff$ are the factors supporting the plaintiff, $F^\defendant$, the defendant, and $o \in \{\plaintiff,\defendant\}$ is the case outcome.
Example of pro-plaintiff factors are that the information was about a product which was unique, in the sense that only the plaintiff manufactured this product ($F^{15}_\plaintiff$), and that the defendant knew that the information was confidential ($F^{21}_\plaintiff$), while some pro-defendant factors are that the plaintiff disclosed the information in negotiations with the defendant ($F^{1}_\defendant)$, and that the plaintiff disclosed the information in a public forum ($F^{27}_\defendant$). For this case study, we use the publicly available 32 cases \cite{DBLP:journals/ail/ChorleyB05,Al_Abdulkarim_2015,DBLP:phd/ethos/Abdulkarim17,Grabmair_thesis}.

Since factors are polarised representations, that is, they indicate a side, we would lose information in treating them simply as features of \oaacbr. It is necessary to incorporate the idea that, if a case is in favour, for instance, of the plaintiff, then removing one of its pro-defendant factors should still decide the same outcome. This is the idea of a case being \emph{constrained}, as typical in the literature of precedential constraint in AI and Law \cite{Horty_2012,Horty_2019,Prakken2020ATM,Prakken_2021}. We accommodate this idea 
by changing the representation of cases. Formally, if a case is $(F^\plaintiff,F^\defendant,o)$, then it yields the following set of $\aacbr$ cases: $\{(F^\plaintiff \cup Y, \plaintiff) \mid Y \subseteq F^\defendant \}$, if $o = \plaintiff$; and $\{(F^\defendant \cup Y, \defendant) \mid Y \subseteq F^\plaintiff \}$, if $o = \defendant$. That is, a single case becomes multiple cases \wrt{} \aacbr. Even though this is not a compact representation (indeed, it is exponential), we only aim to show how $\aacbr$ and cautious monotonicity applies in this domain, not to provide a scalable representation.

In order to give an appropriate comparison of the resulting AF of both \paacbr{} and \caacbr, we remove spikes in the \paacbr{} AF. This makes a more appropriate comparison to \caacbr{} . It turns out that, for this case base, the resulting AF is the same for both \paacbr{} and \caacbr, and shown in Figure \ref{fig:case-study}. However, we show that $\paacbr$ could be manipulated by its violation of cautious monotonicity, while $\caacbr$ cannot.

Consider the following new cases ${\charac{N1} = \caseset[F^{1}_\defendant, F^{18}_\plaintiff, F^{19}_\defendant, F^{21}_\plaintiff, F^{23}_\defendant]{?}}$ and ${\charac{N2}=\caseset[F^{1}_\defendant, F^{10}_\defendant, F^{18}_\plaintiff, F^{19}_\defendant, F^{21}_\plaintiff, F^{23}_\defendant]{?}}$. We can think in terms of two cases an attorney needs to argue, and would like to have a specific outcome, for instance, pro-plaintiff ($\plaintiff$).
For $\charac{N1}$, the predicted $\paacbr$ outcome is $\plaintiff$. For $\charac{N2}$, it is $\defendant$. However, when adding $\case{\charac{\charac{N1}}}{\plaintiff}$ to the casebase, the $\paacbr$ outcome of $\charac{N2}$ then changes to $\plaintiff$. In terms of the domain, a new case, $\charac{N1}$, which brings no different reason (be it distinguishing, change of social values, among others) for change of the case law, indeed changes the system, as proved by the change in $\charac{N2}$.
This implies our attorney in consideration, with no innovation in reasons, could achieve a desired outcome in $\charac{N2}$ by simply presenting it after $\charac{N1}$ is judged. Thus, presenting the cases in different orders would necessarily change the results, even if no new element is introduced, such as considerations of value, policy, or change of legislation.
This is not the case for $\caacbr$. It is straightforward to check that $\case{\charac{N1}}{\plaintiff}$ is non-includable, and thus adding it to the casebase would not change the AF mined from this dataset using $\caacbr$.

\section{Related work and discussion}
Cautious monotonicity is typically discussed in non-monotonic reasoning literature \cite{generalpatterns,DBLP:journals/ai/LehmannM92}, originally presented by \citet{DBLP:conf/nato/Gabbay84} as a reasonable condition for verifying if an allegedly reasoning system is indeed reasoning, that is, a rationality postulate. It is usually presented along cut and cumulativity, which are argued for by \citet{DBLP:journals/ai/KrausLM90} on computational and semantic bases. 

An important element for the occurrence of in\coherence{} in a dataset is the representation of the cases themselves. That is, an insufficiently expressive knowledge representation risks conflating otherwise distinct cases, giving rise to in\coherence{} if they have different outcomes. We should not think this is a matter left entirely to a human user, which would model a dataset by hand. If cases are thought as being originated from previous processes, such as automatic extraction of features by a natural language processing system, as previously done with \aacbr{} \cite{dear-2020}, it is expected that representations could fail in this way, and thus treatment of in\coherence is indeed necessary.

There is a long literature on CBR models for legal reasoning, starting with \citet{DBLP:conf/icail/RisslandA87}, which is surveyed by \citet{DBLP:journals/ail/Bench-Capon17}. The original goal of this literature was to capture the argumentative process, and the influence of abstract argumentation on this literature and on AI and Law in general is surveyed by \citet{DBLP:journals/argcom/Bench-Capon20}. However, its goal has expanded to include prediction of cases \cite{DBLP:conf/icail/BruninghausA03,Grabmair_2017} and to explain predictions \cite{Prakken2020ATM}.
Our treatment of factors (features for and against) in the case study is non-scalable in general, dictated by the restrictions imposed by the structure of \aacbr{}. Factors are subject to much research since their appearance in the work of \citet{DBLP:journals/ai/Aleven03}, and making \aacbr{} more suitable for dealing with them is a topic left for future work, with argumentation-based treatment of them for CBR already occurring in recent research, such as in the work of \citet{Prakken2020ATM}. Another knowledge engineering element also beyond the scope of this work is background knowledge not included in the cases themselves, frequent in the legal CBR literature in the form of a (typically hand-built) domain model, enriching the factor representation \cite{DBLP:journals/ai/Aleven03,DBLP:journals/ail/AshleyB09,Al_Abdulkarim_2016,Grabmair_2017}. For (regular) \aacbr{}, it is assumed that every relevant knowledge engineering aspect is captured by the partial order, case representations, and default argument.

Notwithstanding this literature, the implications of cautious monotonicity (or the lack of it) to legal reasoning has remained largely unexplored, particularly on CBR scenarios.\footnote{\citet{Prakken_1997} mentions it briefly, and cumulativity is critically discussed in non-monotonic reasoning more generally, but not on a CBR or legally motivated context.} We illustrate in Section \ref{sec:case-study} an unexpected consequence of violating it, namely, manipulability of outcomes by leveraging on the order of presentation of new cases. Of course, our analysis is limited and further exploration of the relations between case-based reasoning in law and properties of non-monotonic reasoning systems is still required and left to future work.

\citet{Horty_2012} and \citet{Horty_2019} present formal analyses of precedential constraint. In discussing case base dynamics in the reason model of precedential constraint, \citet{Horty_2012} found out that ``simply following a precedent rule can lead to a change in the law''. One may be led to believe this is an affirmation that an adequate modelling of case law is not cautious monotonic. However, this is not necessarily so. They show that following a rule originated from previous cases may make a decision maker unable to distinguish a new case. That is, merely following a past rule in a case may strengthen the precedential constraint of it, but - and this is the crucial point - we can verify that it would not make a new case previously constrained to an outcome to be constrained to a different outcome. Besides, this effect is only possible if the decision maker is not constrained to an outcome in the changing case (that is, it is still possible to distinguish consistently).

\section{Conclusion}
\label{sec:org500c49c}

We have studied {\wellbehaved} $\paacbr$ frameworks, and proposed a new form of \aacbr, denoted $\caacbr$, which is cautiously monotonic and, as a by-product, cumulative and rationally monotonic.
\ifincludeincoherent
We also show that it results in a principled way of dealing with in{\coherence} in casebases, something which $\paacbr$ lacks.
\fi
Given that $\paacbr$ admits the original \oaacbr\ \cite{DBLP:conf/kr/CyrasST16} as an instance, we have (implicitly) also  defined a cautiously monotonic version thereof.

(Some incarnations of) \aacbr\ have been shown successful empirically in a number of settings \cite{dear-2020}.
The formal properties we have considered in this paper do not necessarily imply better empirical results at the tasks in which $\aacbr$ has been applied. We thus leave for future work an empirical comparison between $\paacbr$ and $\caacbr$. Other issues open for future work are comparisons \wrt\ learnability (such as model performance in the presence of noise), as well as a full complexity analysis of the new model. Also, 
we conjecture that the reduced size of the AF our method generates could possibly have advantages in terms of time and space complexity: we leave investigation of this issue to future work.

\section*{Acknowledgements}
\label{sec:orgab5201e}
We are very grateful to Kristijonas Čyras and Ken Satoh for very valuable discussions, as well as to Alexandre A. A. Almeida, Victor Nascimento and Matheus Müller for reviewing initial drafts of this paper.
We are also grateful for comments by email from Kevin Ashley, as well as for useful comments from anonymous reviewers.
The first author was supported by Capes (Brazil, Ph.D. Scholarship 88881.174481/2018-01).

\bibliographystyle{kr}

\begin{thebibliography}{}

\bibitem[\protect\citeauthoryear{Al-Abdulkarim, Atkinson, and
  Bench-Capon}{2015}]{Al_Abdulkarim_2015}
Al-Abdulkarim, L.; Atkinson, K.; and Bench-Capon, T.
\newblock 2015.
\newblock Evaluating the use of abstract dialectical frameworks to represent
  case law.
\newblock {\em Proceedings of the 15th International Conference on Artificial
  Intelligence and Law}.

\bibitem[\protect\citeauthoryear{Al-Abdulkarim, Atkinson, and
  Bench-Capon}{2016}]{Al_Abdulkarim_2016}
Al-Abdulkarim, L.; Atkinson, K.; and Bench-Capon, T.
\newblock 2016.
\newblock A methodology for designing systems to reason with legal cases using
  abstract dialectical frameworks.
\newblock {\em Artificial Intelligence and Law} 24(1):1–49.

\bibitem[\protect\citeauthoryear{Al{-}Abdulkarim}{2017}]{DBLP:phd/ethos/Abdulkarim17}
Al{-}Abdulkarim, L.
\newblock 2017.
\newblock {\em Representation of case law for argumentative reasoning}.
\newblock Ph.D. Dissertation, University of Liverpool, {UK}.

\bibitem[\protect\citeauthoryear{Aleven}{2003}]{DBLP:journals/ai/Aleven03}
Aleven, V.
\newblock 2003.
\newblock Using background knowledge in case-based legal reasoning: {A}
  computational model and an intelligent learning environment.
\newblock {\em Artif. Intell.} 150(1-2):183--237.

\bibitem[\protect\citeauthoryear{Ashley and
  Br{\"{u}}ninghaus}{2009}]{DBLP:journals/ail/AshleyB09}
Ashley, K.~D., and Br{\"{u}}ninghaus, S.
\newblock 2009.
\newblock Automatically classifying case texts and predicting outcomes.
\newblock {\em Artif. Intell. Law} 17(2):125--165.

\bibitem[\protect\citeauthoryear{Bench{-}Capon}{2017}]{DBLP:journals/ail/Bench-Capon17}
Bench{-}Capon, T. J.~M.
\newblock 2017.
\newblock Hypo's legacy: introduction to the virtual special issue.
\newblock {\em Artif. Intell. Law} 25(2):205--250.

\bibitem[\protect\citeauthoryear{Bench{-}Capon}{2020}]{DBLP:journals/argcom/Bench-Capon20}
Bench{-}Capon, T. J.~M.
\newblock 2020.
\newblock Before and after dung: Argumentation in {AI} and law.
\newblock {\em Argument Comput.} 11(1-2):221--238.

\bibitem[\protect\citeauthoryear{Br{\"{u}}ninghaus and
  Ashley}{2003}]{DBLP:conf/icail/BruninghausA03}
Br{\"{u}}ninghaus, S., and Ashley, K.~D.
\newblock 2003.
\newblock Predicting outcomes of case-based legal arguments.
\newblock In Zeleznikow, J., and Sartor, G., eds., {\em Proceedings of the 9th
  International Conference on Artificial Intelligence and Law, {ICAIL} 2003,
  Edinburgh, Scotland, UK, June 24-28, 2003},  233--242.
\newblock {ACM}.

\bibitem[\protect\citeauthoryear{Chorley and
  Bench{-}Capon}{2005}]{DBLP:journals/ail/ChorleyB05}
Chorley, A., and Bench{-}Capon, T. J.~M.
\newblock 2005.
\newblock {AGATHA:} using heuristic search to automate the construction of case
  law theories.
\newblock {\em Artif. Intell. Law} 13(1):9--51.

\bibitem[\protect\citeauthoryear{Cocarascu \bgroup et al\mbox.\egroup
  }{2020}]{dear-2020}
Cocarascu, O.; Stylianou, A.; Čyras, K.; and Toni, F.
\newblock 2020.
\newblock Data-empowered argumentation for dialectically explainable
  predictions.
\newblock In {\em ECAI 2020 - 24th European Conference on Artificial
  Intelligence, Santiago de Compostela, Spain, 10-12 June 2020}.

\bibitem[\protect\citeauthoryear{Cocarascu, Čyras, and
  Toni}{2018}]{Cocarascu:2018}
Cocarascu, O.; Čyras, K.; and Toni, F.
\newblock 2018.
\newblock Explanatory predictions with artificial neural networks and
  argumentation.
\newblock In {\em 2nd Workshop on XAI at the 27th IJCAI and the 23rd ECAI}.

\bibitem[\protect\citeauthoryear{Dejl \bgroup et al\mbox.\egroup
  }{2021}]{DBLP:conf/atal/DejlHMMSVAL0T21}
Dejl, A.; He, P.; Mangal, P.; Mohsin, H.; Surdu, B.; Voinea, E.; Albini, E.;
  Lertvittayakumjorn, P.; Rago, A.; and Toni, F.
\newblock 2021.
\newblock Argflow: {A} toolkit for deep argumentative explanations for neural
  networks.
\newblock In Dignum, F.; Lomuscio, A.; Endriss, U.; and Now{\'{e}}, A., eds.,
  {\em {AAMAS} '21: 20th International Conference on Autonomous Agents and
  Multiagent Systems, Virtual Event, United Kingdom, May 3-7, 2021},
  1761--1763.
\newblock {ACM}.

\bibitem[\protect\citeauthoryear{Dung}{1995}]{Dung:95}
Dung, P.~M.
\newblock 1995.
\newblock On the acceptability of arguments and its fundamental role in
  nonmonotonic reasoning, logic programming and n-person games.
\newblock {\em Artificial Intelligence} 77(2):321 -- 357.

\bibitem[\protect\citeauthoryear{Dung}{2014}]{DBLP:conf/ecai/Dung14}
Dung, P.~M.
\newblock 2014.
\newblock An axiomatic analysis of structured argumentation for prioritized
  default reasoning.
\newblock In Schaub, T.; Friedrich, G.; and O'Sullivan, B., eds., {\em {ECAI}
  2014 - 21st European Conference on Artificial Intelligence, 18-22 August
  2014, Prague, Czech Republic - Including Prestigious Applications of
  Intelligent Systems {(PAIS} 2014)}, volume 263 of {\em Frontiers in
  Artificial Intelligence and Applications},  267--272.
\newblock {IOS} Press.

\bibitem[\protect\citeauthoryear{Dung}{2016}]{Dung_2016}
Dung, P.~M.
\newblock 2016.
\newblock An axiomatic analysis of structured argumentation with priorities.
\newblock {\em Artificial Intelligence} 231:107–150.

\bibitem[\protect\citeauthoryear{Gabbay}{1984}]{DBLP:conf/nato/Gabbay84}
Gabbay, D.~M.
\newblock 1984.
\newblock Theoretical foundations for non-monotonic reasoning in expert
  systems.
\newblock In Apt, K.~R., ed., {\em Logics and Models of Concurrent Systems -
  Conference proceedings, Colle-sur-Loup (near Nice), France, 8-19 October
  1984}, volume~13 of {\em {NATO} {ASI} Series},  439--457.
\newblock Springer.

\bibitem[\protect\citeauthoryear{Grabmair}{2016}]{Grabmair_thesis}
Grabmair, M.
\newblock 2016.
\newblock {\em Modeling purposive legal argumentation and case outcome
  prediction using argument schemes in the value judgment formalism}.
\newblock Ph.D. Dissertation, University of Pittsburgh, {USA}.

\bibitem[\protect\citeauthoryear{Grabmair}{2017}]{Grabmair_2017}
Grabmair, M.
\newblock 2017.
\newblock Predicting trade secret case outcomes using argument schemes and
  learned quantitative value effect tradeoffs.
\newblock {\em Proceedings of the 16th edition of the International Conference
  on Articial Intelligence and Law}.

\bibitem[\protect\citeauthoryear{Horty and Bench-Capon}{2012}]{Horty_2012}
Horty, J.~F., and Bench-Capon, T. J.~M.
\newblock 2012.
\newblock A factor-based definition of precedential constraint.
\newblock {\em Artificial Intelligence and Law} 20(2):181–214.

\bibitem[\protect\citeauthoryear{Horty}{2019}]{Horty_2019}
Horty, J.
\newblock 2019.
\newblock Reasoning with dimensions and magnitudes.
\newblock {\em Artificial Intelligence and Law} 27(3):309–345.

\bibitem[\protect\citeauthoryear{Hunter}{2010}]{DBLP:conf/comma/Hunter10}
Hunter, A.
\newblock 2010.
\newblock Base logics in argumentation.
\newblock In Baroni, P.; Cerutti, F.; Giacomin, M.; and Simari, G.~R., eds.,
  {\em Computational Models of Argument: Proceedings of {COMMA} 2010, Desenzano
  del Garda, Italy, September 8-10, 2010}, volume 216 of {\em Frontiers in
  Artificial Intelligence and Applications},  275--286.
\newblock {IOS} Press.

\bibitem[\protect\citeauthoryear{Kenny and
  Keane}{2019}]{DBLP:conf/ijcai/KennyK19}
Kenny, E.~M., and Keane, M.~T.
\newblock 2019.
\newblock Twin-systems to explain artificial neural networks using case-based
  reasoning: Comparative tests of feature-weighting methods in {ANN-CBR} twins
  for {XAI}.
\newblock In Kraus, S., ed., {\em Proceedings of the Twenty-Eighth
  International Joint Conference on Artificial Intelligence, {IJCAI} 2019,
  Macao, China, August 10-16, 2019},  2708--2715.
\newblock ijcai.org.

\bibitem[\protect\citeauthoryear{Kraus, Lehmann, and
  Magidor}{1990}]{DBLP:journals/ai/KrausLM90}
Kraus, S.; Lehmann, D.; and Magidor, M.
\newblock 1990.
\newblock Nonmonotonic reasoning, preferential models and cumulative logics.
\newblock {\em Artif. Intell.} 44(1-2):167--207.

\bibitem[\protect\citeauthoryear{Lehmann and
  Magidor}{1992}]{DBLP:journals/ai/LehmannM92}
Lehmann, D., and Magidor, M.
\newblock 1992.
\newblock What does a conditional knowledge base entail?
\newblock {\em Artif. Intell.} 55(1):1--60.

\bibitem[\protect\citeauthoryear{Makinson}{1994}]{generalpatterns}
Makinson, D.
\newblock 1994.
\newblock General patterns in nonmonotonic reasoning.
\newblock In Gabbay, D.~M.; Hogger, C.~J.; and Robinson, J.~A., eds., {\em
  Handbook of Logic in Artificial Intelligence and Logic Programming - Volume 3
  - Nonmonotonic Reasoning and Uncertain Reasoning},  35--110.
\newblock Oxford University Press.

\bibitem[\protect\citeauthoryear{Modgil and
  Caminada}{2009}]{DBLP:books/sp/09/ModgilC09}
Modgil, S., and Caminada, M.
\newblock 2009.
\newblock Proof theories and algorithms for abstract argumentation frameworks.
\newblock In Simari, G.~R., and Rahwan, I., eds., {\em Argumentation in
  Artificial Intelligence}. Springer.
\newblock  105--129.

\bibitem[\protect\citeauthoryear{Nugent and
  Cunningham}{2005}]{DBLP:journals/air/NugentC05}
Nugent, C., and Cunningham, P.
\newblock 2005.
\newblock A case-based explanation system for black-box systems.
\newblock {\em Artif. Intell. Rev.} 24(2):163--178.

\bibitem[\protect\citeauthoryear{Paulino-Passos and Toni}{2020}]{caacbr}
Paulino-Passos, G., and Toni, F.
\newblock 2020.
\newblock Cautious monotonicity in case-based reasoning with abstract
  argumentation.
\newblock 18th International Workshop On Non-Monotonic Reasoning.

\bibitem[\protect\citeauthoryear{Prakken \bgroup et al\mbox.\egroup
  }{2015}]{trevor}
Prakken, H.; Wyner, A.~Z.; Bench{-}Capon, T. J.~M.; and Atkinson, K.
\newblock 2015.
\newblock A formalization of argumentation schemes for legal case-based
  reasoning in {ASPIC+}.
\newblock {\em J. Log. Comput.} 25(5):1141--1166.

\bibitem[\protect\citeauthoryear{Prakken}{1997}]{Prakken_1997}
Prakken, H.
\newblock 1997.
\newblock Logical tools for modelling legal argument.
\newblock {\em Law and Philosophy Library}.

\bibitem[\protect\citeauthoryear{Prakken}{2020}]{Prakken2020ATM}
Prakken, H.
\newblock 2020.
\newblock A top-level model of case-based argumentation for explanation.
\newblock Proceedings of the ECAI 2020 Workshop on Dialogue, Explanation and
  Argumentation for Human-Agent Interaction (DEXA HAI 2020).

\bibitem[\protect\citeauthoryear{Prakken}{2021}]{Prakken_2021}
Prakken, H.
\newblock 2021.
\newblock A formal analysis of some factor- and precedent-based accounts of
  precedential constraint.
\newblock {\em Artificial Intelligence and Law}.

\bibitem[\protect\citeauthoryear{Rago \bgroup et al\mbox.\egroup
  }{2021}]{DBLP:journals/ai/RagoCBLT21}
Rago, A.; Cocarascu, O.; Bechlivanidis, C.; Lagnado, D.~A.; and Toni, F.
\newblock 2021.
\newblock Argumentative explanations for interactive recommendations.
\newblock {\em Artif. Intell.} 296:103506.

\bibitem[\protect\citeauthoryear{Rissland and
  Ashley}{1987}]{DBLP:conf/icail/RisslandA87}
Rissland, E.~L., and Ashley, K.~D.
\newblock 1987.
\newblock A case-based system for trade secrets law.
\newblock In {\em Proceedings of the First International Conference on
  Artificial Intelligence and Law, {ICAIL} '87, Boston, MA, USA, May 27-29,
  1987},  60--66.
\newblock {ACM}.

\bibitem[\protect\citeauthoryear{Shih, Choi, and Darwiche}{2019}]{Shih_19}
Shih, A.; Choi, A.; and Darwiche, A.
\newblock 2019.
\newblock Compiling bayesian network classifiers into decision graphs.
\newblock In {\em The Thirty-Third {AAAI} Conference on Artificial
  Intelligence, {AAAI} 2019, The Thirty-First Innovative Applications of
  Artificial Intelligence Conference, {IAAI} 2019, The Ninth {AAAI} Symposium
  on Educational Advances in Artificial Intelligence, {EAAI} 2019, Honolulu,
  Hawaii, USA, January 27 - February 1, 2019},  7966--7974.

\bibitem[\protect\citeauthoryear{Čyras and
  Toni}{2015}]{DBLP:conf/tafa/CyrasT15}
Čyras, K., and Toni, F.
\newblock 2015.
\newblock Non-monotonic inference properties for assumption-based
  argumentation.
\newblock In Black, E.; Modgil, S.; and Oren, N., eds., {\em Theory and
  Applications of Formal Argumentation - Third International Workshop, {TAFA}
  2015, Buenos Aires, Argentina, July 25-26, 2015, Revised Selected Papers},
  volume 9524 of {\em Lecture Notes in Computer Science},  92--111.
\newblock Springer.

\bibitem[\protect\citeauthoryear{Čyras and
  Toni}{2016}]{DBLP:journals/corr/CyrasT16}
Čyras, K., and Toni, F.
\newblock 2016.
\newblock Properties of {ABA+} for non-monotonic reasoning.
\newblock {\em CoRR} abs/1603.08714.

\bibitem[\protect\citeauthoryear{Čyras \bgroup et al\mbox.\egroup
  }{2019}]{DBLP:journals/eswa/CyrasBGTDTGH19}
Čyras, K.; Birch, D.; Guo, Y.; Toni, F.; Dulay, R.; Turvey, S.; Greenberg, D.;
  and Hapuarachchi, T.
\newblock 2019.
\newblock Explanations by arbitrated argumentative dispute.
\newblock {\em Expert Syst. Appl.} 127:141--156.

\bibitem[\protect\citeauthoryear{Čyras, Satoh, and
  Toni}{2016a}]{DBLP:conf/kr/CyrasST16}
Čyras, K.; Satoh, K.; and Toni, F.
\newblock 2016a.
\newblock Abstract argumentation for case-based reasoning.
\newblock In {\em {KR} 2016},  549--552.

\bibitem[\protect\citeauthoryear{Čyras, Satoh, and
  Toni}{2016b}]{DBLP:conf/comma/CyrasST16}
Čyras, K.; Satoh, K.; and Toni, F.
\newblock 2016b.
\newblock Explanation for case-based reasoning via abstract argumentation.
\newblock In {\em Proceedings of {COMMA} 2016},  243--254.

\end{thebibliography}

\appendix
\newpage

\section{Proofs of theorems}
\subsection{Proof of Theorem \ref*{theo:compl}}
Completeness and consistency here are immediate consequences from the fact that $\learn$ is a total function. This is in line on how classifiers are typically considering: returning exactly one output for each input. The remaining properties, which are the non-monotonic ones, are essentially a collapse caused by consistency.

\begin{proof}
  \begin{enumerate}
\item  By definition of $\infers_{\learn}$, directly from the totality of $\learn$.
\item  By definition of $\infers_{\learn}$, since $\learn$ is a function.
\item  {Assume $\infers_{\learn}$ cautiously monotonic, and let $D \infers_{\learn} p$ and $D \cup \{p\} \infers_{\learn} q$, for $p,q \in \lang$. By completeness, either $D \infers_{\learn} q$ or $D \infers_{\learn} \neg q$ (here $\neg q= r$ if $q = \neg r$, and $\neg r$ if $q=r$). In the first case we are done. Suppose the second case holds. Since $D \infers_{\learn} p$, by cautious monotonicity $D \cup \{p\} \infers_{\learn} \neg q$. But then $D \infers_{\learn} q$ and $D \infers_{\learn} \neg q$, which is absurd since $\infers_{\learn}$ is consistent. Therefore $D \not \infers_{\learn} \neg q$, and then $D \infers_{\learn} q$. The converse can be proven analogously.}
  
\item {Trivial from 3.}
\item Since $\infers_{\learn}$ is complete, $D \not \infers_{\learn} \neg p$ implies $D \infers_{\learn} p$, and thus rational monotonicity is reduced to cautious monotonicity.
\end{enumerate}
\end{proof}

\subsection{Proof of Theorem \ref*{theo:unique-concise}}
In order to prove Theorem \ref*{theo:unique-concise}, we first require a useful lemma.
\subsubsection{Coinciding predictions.}
This lemma identifies a ``core'' in the casebase for the purposes of outcome prediction: this amounts to all past cases that are less (or equally) specific than the new case for which the prediction is sought. {In other words, irrelevant cases in the casebase do not affect the prediction in regular AFs.}

\begin{lemma} \label{lemma:lesser}
  Let $D_1$ and $D_2$ be two datasets. Let $\newcasecharac \in X$ be a characterisation, and ${D_i}_\newcasecharac = \{\casei \in D_i \mid \casei \pleq \newcasecharac\}$ {for $i=1,2$}. If ${D_1}_\newcasecharac = {D_2}_\newcasecharac$, then {$\paacbr(D_1,\newcasecharac) = \paacbr(D_2,\newcasecharac)$} (that is, \paacbr\ predicts the same outcome for $\newcasecharac$ given the two datasets).

\end{lemma}
\begin{proof}
  {For $i = 1,2$, let $AF_i = \aaFtwo{D_i}{\newcasecharac}$ and the grounded extensions be $\groundext_i = \bigcup_{j \geqslant 0} G^i_j$.}
We will prove that $\forall j: G^1_j \subseteq G^2_{j+1}$ and $G^2_j \subseteq G^1_{j+1}$, and this allows us to prove that $\groundext_1 = \groundext_2$, which in turn implies the outcomes are the same.
Here we consider only $G^1_j \subseteq G^2_{j+1}$, as the other case is entirely symmetric.
  By induction on $j$:

  \begin{itemize}
  \item For the base case $j = 0$:

    If $G^1_0 \subseteq G^2_{0}$, we are done, since we always have that $G^i_j \subseteq G^i_{j+1}$. If not, there is a $\casei \in G^1_0 \setminus G^2_0$.
  Since $\casei \in G^1_0$, it is relevant to $\newcasecharac$, and thus $\casei \pleq \newcasecharac$, which in turn implies that $\casei \in D_2$, since ${D_1}_\newcasecharac = {D_2}_\newcasecharac$.
  
  On the other hand, as $\casei \not\in G^2_0$, there is a case $\caseii \in AF_2$ such that $\caseii \attacks \casei$. However, $\casei \not \in AF_1$, otherwise $\casei$ would be attacked in $AF_1$ and thus not in $G^1_0$. But then, since ${D_1}_\newcasecharac = {D_2}_\newcasecharac$, this means that $\caseii \not \pleq \newcasecharac$. Finally, this means that $\newcasearg \attacks \caseii$, and thus $G^2_0$ defends it. Therefore, $\caseii \in G^2_1$, what we wanted to prove.

  \item For the induction step, from $j$ to $j+1$:

  Again, if $G^1_{j+1} \!\subseteq\! G^2_{j+1}$, we are done. If not, there is a $\casei \in G^1_{j+1} \setminus G^2_{j+1}$. Again we can check that this implies that $\casei \in D_2$. Now, since $\casei \in G^1_{j+1}$, then $G^1_{j}$ defends it. But now, by inductive hypothesis, $G^1_{j} \subseteq G^2_{j+1}$. Therefore, $G^2_{j+1}$ also defends $\casei$, which implies that $\casei \in G^2_{j+2}$,as we wanted.\footnote{In abstract argumentation it can be verified that, if $E\subseteq \Args$ defends an argument $\caseiii$, and $E \subseteq E'$, then $E'$ also defends $\caseiii$.} This concludes the induction.
\end{itemize}

  To conclude, we can now see that $\groundext_1 = \groundext_2$, since, once more without loss of generality, if we consider $\casei \in \groundext_1$, by definition of $\groundext_1$ there is a $j$ such that $\casei \in G^1_j$. But since $G^1_j \subseteq G^2_{j+1}$, $\casei \in \groundext_2$. This proves that $\groundext_1 \subseteq \groundext_2$. The converse can be proven analogously. \qedhere
\end{proof}

\subsubsection{Uniqueness of concise subsets.}
We are now ready to prove Theorem \ref*{theo:unique-concise}.
  \newcommand{\mycase}{\ensuremath{(x,y)}}
  \newcommand{\opposite}{\ensuremath{(x,\bar{y})}}

\begin{proof}[Theorem \ref*{theo:unique-concise}]
  By contradiction, let $D'$ and $D''$ be distinct concise subsets of $D$.
  Let then $(x,y) \in (D' \setminus D'') \cup (D'' \setminus D')$ (the symmetric difference between $D'$ and $D''$) such that $(x,y)$ is {$\pleq$-}minimal in this set. It exists since the sets are different (there is at least one element in one of the sets and not in the other) and this set is finite. Then we have that:
  \begin{align*} \label{eq:smaller-equal}
    &\{(x',y') \in D' \mid (x',y') \pl (x,y) \} =& \\
    &\{(x',y') \in D'' \mid (x',y') \pl (x,y) \},&
  \end{align*}
  otherwise $(x,y)$ would not be minimal in the symmetrical difference.

  Without loss of generality, consider $(x,y) \in D'$ (and thus $(x,y) \not \in D''$). Let $\bar{y}$ be the opposite outcome to $y$, that is, $\bar{y} \in Y$ and $\bar{y} \neq y$. It is straightforward to check that $\opposite{} \not \in D'$, otherwise $D'$ would not be concise, since one of \mycase{} and \opposite{} would not be {\includable}.
  
  If we have that $\opposite{} \not \in D''$, then with the equality above we conclude that
  \begin{align*}
    &\{\caseii \in D' \setminus \{\mycase\} \mid \charac{\caseii} \pleq x \} = \{\caseii \in D'' \setminus \{\mycase\} \mid \charac{\caseii} \pleq x\}\\ &\text{and}\\
    &\{\caseii \in D' \cup \{\mycase\} \mid \charac{\caseii} \pleq x \} = \{\caseii \in D'' \cup \{\mycase\} \mid \charac{\caseii} \pleq x \}\text{.}
  \end{align*}
Thus, \mycase{} is {\includable} to $D'$ if and only if it is {\includable} to $D''$, which contradicts $D''$ being concise, since $D''$ would be lacking a {\includable} case, as $\mycase \not \in D''$.
  
  Now let us consider the situation in which $\opposite{} \in D''$. Suppose that $y = \defoutcome$, that is, it is the default outcome. We can see that $D' \setminus \{\opposite\} = D'$ and $D' \not \inferspaacbr \opposite$. Still, $D' \cup \{\mycase\} \inferspaacbr \opposite$, since it would create an in\coherence, making the outcome be non-default. Therefore $\opposite$ is {\includable} to $D'$, yet not a member of it, contradicting $D'$ being concise. The case for $\bar{y} = \defoutcome$ instead is analogous, checking that $\mycase$ is {\includable} to $D''$ yet not a member of it.
  Therefore, for any possibility, the assumption of more than one concise subset of $D$ leads to an absurdity, and thus we conclude there is at most one concise subset.
  \qedhere \end{proof}

\subsection{Proof of Theorem \ref*{theo:existence-concise}}
Again, we will first need a lemma.

\subsubsection{Addition of new cases.}
\label{sec:org3312bc4}
{The next result characterises the set of past cases/arguments attacked  when the dataset is extended with a new labelled case/argument. In particular,
this result compares the effect of predicting the outcome of some $N_2$
from $D$ alone and from $D$ extended with $\case{N_1}{o_1}$, when there is no case in $D$ with characterisation $N_1$ already and moreover $D$ is \coherent.}

{This result is interesting in its own right as it shows that, any argument attacked by the ``newly added'' case $\case{N_1}{o_1}$ is easily identified in the sets $G_0$ and $G_1$ in the grounded extension $\groundext$, being sufficient to check those rather than the entire casebase $D$. Notice that we require $D$ to be \coherent.}

\begin{lemma}
\label{lemma:attacked-entering-case}
Let $D$ be \coherent, $N_1, N_2 \in X$, $o_1 \in Y$, and suppose that there is no case in $D$ with characterisation $N_1$. Consider $AF_1 = \aaFtwo{D}{N_1}$ and $AF_2 = 
	\aaFtwo{D \cup \{\case{N_1}{o_1}\}}{N_2}$. Finally, let $\groundext(AF_1) $and $\groundext(AF_{2})$ be the respective grounded extensions.
Let $\caseii \in D$ be such that $\case{N_1}{o_1} \attacks \caseii$ in $AF_{2}$. Then, 
	\begin{enumerate}
        \item 	for every {$\caseiii$ that attacks}  $\caseii$ in $AF_1$,
          $N_1 \not\sim \caseiii$ (that is,  $\caseiii$ is irrelevant to $N_1$ {and, by regularity, $N_1 \not\pgeq \caseiii$)};
		\item in $AF_1$, $\case{N_1}{?}$ defends $\caseii$; 
		\item $\caseii \in \groundext(AF_1)$ and,  for $\groundext(AF_1)=\bigcup_{i \geqslant 0} G_i$, $\caseii$ is either in $G_0$ (that it, it is unattacked), or in $G_1$.
                \item {For every $\casev = \fullcase{\casev} \in D$ such that $\case{N_1}{?}$ defends ${\casev}$ in $AF_1$, if $\outcome{\casev} \neq o_1$, then, in $AF_2$, $\case{N_1}{o_1} \attacks \casev$.}
	\end{enumerate}
\end{lemma}

\begin{proof} 
  \begin{enumerate}
  \item 
    Let $\caseii=\fullcase{\caseii}$. From the definition of attack:
    (i) $N_1 \pg \charac{\caseii} $ (this is strict due to \coherence{} and since no case in $D$ has characterisation $N_1$),
    (ii) $o_1 \neq \outcome{\caseii}$, and 
    (iii) there is no $(\charac\casei, \outcome{x})$ such that $\outcome{x} = o_1$ and $N_1 \pg \charac\casei \pg \charac{\caseii}$. \label{conciseness-req}
    Consider $\caseiv = (\charac{\caseiv}, \outcome{\caseiv})$ such that $\caseiv$ attacks $\caseii$ in $AF_1$ (if there is no such $\caseiv$ then the result trivially holds).
		  {Assume by contradiction that} $\caseiv$ is relevant to $N_1$. Then {by regularity} $N_1 \pgeq \charac{\caseiv}$. But since $D$ is {\coherent} and $\case{N_1}{o_1}\not \in D$, $\caseiv$ and $N_1$ are distinct, and thus $N_1 \pg \charac{\caseiv}$. As $\caseiv$ attacks $\caseii$, $\outcome{\caseiv} \neq \outcome{\caseii}$, but this in turn implies that $\outcome{\caseiv} = o_1$, since $\case{N_1}{o_1}$ also attacks $\caseii$, in $AF_{2}$. But then $N_1 \pg \charac{\caseiv} \pg \charac\caseii$, with $\outcome{\caseiv} = o_1$. This contradicts {\concision} of the attack between $(N_1, o_1)$ and $\caseii$. Therefore, $\caseiv$ is not relevant to $N_1$, as we wanted to prove.
  \item {Trivially true, by} 1 {(as,} if $\caseiv$ is an attacker $\caseii$, then $N_1 \not\sim \caseiv$; but then $\case{N_1}{?} \attacks \caseiv$).
  \item Trivially true, by 2.
    \item {Since, in $AF_1$, $\case{N_1}{?}$ defends ${\casev}$, then any attacker $\caseiv$ of $\casev$ is irrelevant to $N_1$, and by regularity, $N_1 \not\pgeq \caseiv$. Thus {\concision} is satisfied. Requirement 1 is the hypothesis and requirement 2 is satisfied since $\case{N_1}{?}$ defends ${\casev}$ in $AF_1$.} \qedhere
  \end{enumerate}
\end{proof}

Now, we can show the theorem. We will first prove that Algorithm \ref*{algo:simpleadd} is correct.
\begin{theorem}[Correctness of Algorithm \ref*{algo:simpleadd}] \label{theo:correctness-add}
  Every execution of $simple\_add((Args,$ $\attacks), next\_case)$ (Algorithm \ref*{algo:simpleadd}) in Algorithm \ref*{algo:cumulaacbr} correctly returns  $\aaFone{Args \cup \{next\_case\}}$.
\end{theorem}

\begin{proof}[Proof sketch]
  This is essentially a consequence of Lemma \ref{lemma:attacked-entering-case}. We know that there will never be an argument in $\Args$ with the same characterisation as $next\_case$, since they will occur in the same stratum, thus the lemma applies. The lemma guarantees that Algorithm \ref*{algo:simpleadd} adds all attacks that need to be added {and only those}. 
  Finally, we need to check that it will never be necessary to remove an attack. This is true due to the {\concision}, and since arguments are added following the partial order.
Therefore the only modifications on the set of attacks are the ones in $simple\_add$. \qedhere
\end{proof}

\begin{proof}[Theorem \ref*{theo:existence-concise}]
  \begin{enumerate}
  \item Every case in $D'$ is surprising {\wrt} $D'$.
    
  \newcommand{\partialD}{\ensuremath\hat{D'}}
  Algorithm \ref*{algo:cumulaacbr} explicitly only adds a case if it is {\surprising}. Since {\coherence} is maintained at every step of the algorithm and it proceeds by following the partial order, one can see that after a case $\fullcase{\casei}$ is added, the outcome predicted for a new case $\newcasearg[\casei]$ is always $\outcome{\casei}$, by Lemma \ref{lemma:lesser}. Thus every case in $D'$ is surprising {\wrt} $D'$.

\item If $D$ has a concise subset, it is found.
  
    Now, suppose that $D$ has a concise subset $D''$. We already know by Theorem \ref*{theo:unique-concise} that it is the unique concise subset. Now we need to show that $D' = D''$. We will show by contradiction, in a similar fashion to the proof of Theorem \ref*{theo:unique-concise}:
  
  Let then $\mycase \in (D' \setminus D'') \cup (D'' \setminus D')$ (the symmetric difference between $D'$ and $D''$) such that $\mycase$ is {$\pleq$-}minimal in this set. It exists since the sets are different (there is at least one element in one of the sets and not in the other) and this set is finite.
  Then we have that:
  \begin{align*} 
    &\{(x',y') \in D' \mid (x',y') \pl \mycase \} =& \\
    &\{(x',y') \in D'' \mid (x',y') \pl \mycase \}&
  \end{align*}
  otherwise $\mycase$ would not be minimal in the symmetrical difference.

  However, every concise set is \coherent. Besides, at every step, the currently selected subset of $D$ is also \coherent. Therefore in the equality above we can replace $\pl$ with $\pleq$, and by applying Lemma \ref{lemma:lesser} we will conclude that either $\mycase$ is {\includable} {\wrt} both $D'$ and $D''$ or to neither. Since $D''$ is concise, it must be {\includable} {\wrt} both, and thus $\mycase \in D''$.

  Now consider the moment when $\mycase$ was considered in Algorithm \ref*{algo:cumulaacbr}. At that point, every case in $\{(x',y') \in D' \mid (x',y') \pl \mycase \}$ was already in $\Args$, since Algorithm \ref*{algo:cumulaacbr} follows the partial order. Thus it was considered to be {\includable} if and only if it is {\includable} {\wrt} $D'$. However, we already know that it is {\includable} {\wrt} $D'$. Thus it was considered to be {\includable} at that point. Thus, it was added, and then $\mycase in D'$. This is absurd, since we defined $\mycase$ as an element of $\mycase \in (D' \setminus D'') \cup (D'' \setminus D')$.

\item A {\coherent} dataset has a concise subset
  
  We will show this by showing that, if the input dataset is \coherent, then the dataset underpinning the AF resulting from Algorithm \ref*{algo:cumulaacbr} is concise.

In order to prove that, for the returned $\Args$,  $\Args\setminus\{\defcase\}$ is concise, we just need to prove that at the end of each loop $\Args\setminus\{\defcase\}$ is concise \wrt\ the set of all seen examples. 
  
  As  the base case, before the loop is entered, this is clearly the case, as the only seen argument is the default.

  As the induction step, we know that every case previously added is still {\includable}, since the new cases added are not less specific than them according to the partial order, and thus by Lemma \ref{lemma:lesser} their prediction is not changed, that is, they keep being {\includable}. The same is true for every case previously not added: adding more cases afterwards does not change their prediction. For the cases added at this new iteration, by definition the {\includable} ones are added and the not {\includable} ones are not.
Regarding the order in which cases of the same stratum are added, each of the {\includable} cases will be included and the not {\includable} ones will not be. It can be seen that the order is irrelevant as, since they are all $\pleq$-minimal and the dataset is {\coherent}, they are incomparable, so each case in the list is irrelevant with respect to the other.
Thus, for every case seen until this point, it is in the AF iff it is {\includable}. As this is true for every iteration, it is true for the final, returned AF. \qedhere
  \end{enumerate}
\end{proof}

\subsection{\Coherence{} in \caacbr{} - Proof of Theorem \ref*{theo:caacbr-coherent}}
\newcommand{\ci}{\case{\charac{\casei}}{\defoutcome}}
\newcommand{\cii}{\case{\charac{\casei}}{\nondefoutcome}}
\begin{proof}
  We will $D$ is in\coherent{}, and thus there are at least two cases with the same characterisation and different outcomes, namely, $\ci$ and $\cii$. Then during the execution of Algorithm \ref*{algo:cumulaacbr}, $\ci$ and $\cii$ occur in the same stratum, that is, they are in \texttt{stratum} in the same \texttt{while} loop. This is clear since they share the same characterisation, and $\ci$ is $\pleq\!\!\text{-minimal}$ if and only if $\cii$ is. At this point, there is a current AF, and the predicted outcome for $\charac{\casei}$ is either $\defoutcome$ or $\nondefoutcome$. If the former, $\cii$ is added to \texttt{to\_add}, while $\ci$ is not, and if the latter, the opposite happens. Thus only one of them is indeed added to the AF, and thus the dataset underpinning the AF when Algorithm \ref*{algo:cumulaacbr} terminates is \coherent.
\end{proof}

\subsection{Spikes - Proof of Theorems \ref*{theo:remove-spike} and \ref*{theo:no-spikes}}
\subsubsection{Nearest cases.}
Before proving those theorems, an important property of the predictions of $\paacbr$ in relation to the  ``most similar'' (or \emph{nearest}) cases to the new case needs to be proved. \gppnew{This result is already presented in our previous (non-archival) work, but we repeat it for convenience of the reader \cite{caacbr}}.
The result is that when these nearest cases all agree on an outcome, the prediction is necessarily this outcome. It generalises \cite[Proposition 2]{DBLP:conf/kr/CyrasST16} {in two ways:} by considering the entire set of nearest cases, instead of requiring a unique nearest case, for $\paacbr$, instead of its instance \oaacbr. As in \cite{DBLP:conf/kr/CyrasST16}, we prove this property for \coherent\ casebases.

We first define the notion of nearest case.

\begin{definition}
A case $\fullcase{\casei} \in D$ is \defemph{nearest to $\newcasecharac$} iff $\charac\casei \pleq \newcasecharac$ and 
	it is maximally so, that is, 
	there is no $\fullcase\caseii\in D$ such that $\charac\casei \pl \charac\caseii \pleq \newcasecharac$.
\end{definition}

  \begin{theorem} \label{theo:nearest_neighbours}
If {$D$ is {\coherent} and} every nearest case to $\newcasecharac$ is of the form $\case{\charac\casei}{o}$ {for some outcome $o\in Y$} (that is, all nearest  cases to the new case agree on the same outcome), then $\paacbr(D,\newcasecharac)=o$ (that is, the outcome for $\newcasecharac$ is $o$).
\end{theorem}

  \begin{proof} Let $\groundext$ be the grounded extension of \aaDN. An outline of the proof is as follows:
	  
\begin{enumerate}
\item \label{nn:ind} We will first prove that each argument in $\groundext$ is either \(\newcasearg\) or of the form \(\case{\charac\caseii}{o}\) (that is, agreeing in outcome with all nearest cases).

\item Then we will prove that if \(o = \nondefoutcome\) (that is, \(o\) is the non-default outcome), then \(\defcase\not\in\groundext\) {(and thus $\paacbr(D,\newcasecharac)=\nondefoutcome$,
	as envisaged by the theorem)}.

\item \label{nn:def} Finally, by using the fact that $\aaDN$ is well-founded {(given that $D$ is \coherent)},
and thus $\groundext$ is also stable, we will prove that if \(o = \defoutcome\) (that is, \(o\) is the default outcome), then \(\defcase\in \groundext\) {(and thus $\paacbr(D,\newcasecharac)=\defoutcome$, as envisaged by the theorem)}.
\end{enumerate}
We will now prove 1-3.

\begin{enumerate}
	\item By definition $\groundext = \bigcup_{i \geqslant 0} G_i$. 
		We prove by induction that, for every $i$,  each argument in $G_i$ is either \(\newcasearg\) or of the form \(\case{\charac\caseii}{o}\).
		{Then, given that each element of  $\groundext$ belongs to some $G_i$, the property holds for $\groundext$}.

  \begin{enumerate}
	  \item For the base case, consider \(G_0\).  \(\newcasearg\) and all nearest cases are {unattacked, and thus in $G_0$} {(notice how this requires the AF to be regular, otherwise nearest cases could be irrelevant)}. \(G_0\) may however contain further unattacked cases. Let \(\caseii = \fullcase{\caseii}\) be such a case. If \(\newcasecharac \not \pgeq \charac{\caseii}\), then $\defcase \not\sim \caseii$ and thus \(\newcasearg\) attacks \(\caseii\), contradicting that \(\caseii\) in unattacked. So \(\charac{\caseii} \pleq \newcasecharac\). As \(\caseii\) is not a nearest case, there is a nearest case \(\casei = \fullcase{\casei}\) such that 
\(\charac{\caseii} \pl \charac{\casei}\). 
By contradiction, assume 
		  \(\outcome{\caseii} \neq o\).  Let \(\Gamma = \{\gamma \in \Args\ |\ \gamma = \fullcase{\gamma}\), \(\charac{\caseii} \pl \charac{\gamma} \pleq \charac{\casei}\) and \(\outcome{\gamma} = o \}\). Notice that \(\Gamma\) is non-empty, as \(\casei \in \Gamma\). $\Gamma$ is the set of ``potential attackers'' of \(\caseii\), but only {$\pleq$-minimal} arguments in $\Gamma$ do actually attack \(\caseii\).  Let \(\caseiv\) be such a {$\pleq$-minimal} element of \(\Gamma\).\footnote{Note that \(\caseiv\) is guaranteed to exist, as \(\Gamma\) is non-empty and otherwise we would be able to build {an arbitrarily long chain of (distinct) {arguments, decreasing} \wrt\ $\pl$. However this would allow a chain with more elements than the cardinality of \(\Gamma\), which is absurd.}}
By construction, \(\caseiv\) attacks \(\caseii\). Thus \(\caseii\) is attacked and not in \(G_0\), a contradiction.
Hence, \(\outcome{\caseii} = o\), as required.

\item For the inductive step, let us assume that the property holds for a generic \(G_i\), and let us prove it for \(G_{i+1}\).
Let \(\caseii= \fullcase{\caseii} \in G_{i+1} \setminus G_i\) (if \(\caseii \in G_i\), the property holds by the induction hypothesis).
\(\newcasearg\) does not attack \(\caseii\), as otherwise \(\caseii\) would not be defended by \(G_i\), as \(G_i\) is conflict-free. Thus, once again, as \(\caseii\) is not a nearest case, there is a nearest case \(\casei=\fullcase{\casei}\) such that \(\charac{\caseii} \pl \charac{\casei}\).
		  Again, assume that \(\outcome{\caseii} \neq o\). Then let \(\Gamma = \{\gamma \in \Args\ |\ \gamma = (\charac{\gamma}, \outcome{\gamma})\), \(\charac{\caseii} \pl \charac{\gamma} \pleq \charac{\casei}\) and \(\outcome{\gamma} = o \}\), with \(\caseiv\) a {$\pleq$-minimal element} of \(\Gamma\). 
		  Then \(\caseiv\) attacks \(\caseii\). However, as \(G_i\) defends \(\caseii\), there is then \(\casev \in G_i\) such that \(\casev\) attacks \(\caseiv\). 
By inductive hypothesis, \(\casev\) is either \(\newcasearg\) or \(\casev = (\charac{\casev}, o)\). The first option is not possible, as \(\caseiv \in \Gamma\), and thus \(\charac{\caseiv} \pleq \charac{\casei}\), and of course \(\charac{\casei} \pleq \newcasecharac\). Thus, \(\charac{\caseiv} \pleq \newcasecharac\) and is thus not attacked by $\newcasearg$. This means that \((\charac{\casev}, o)\) attacks \(\caseiv = (\charac{\caseiv}, \outcome{\caseiv})\). But this is absurd as well, as \(\caseiv \in \Gamma\) and thus \(\outcome{\caseiv} = o = \outcome{\casev}\).
Therefore, our assumption that \(\outcome{\caseii} \neq o\) was false, that is, \(\outcome{\caseii} = o\), as required.

\end{enumerate}

\item If \(o = \nondefoutcome\), the default argument \(\defcase\) is not in $\groundext$, since we have just proven that all arguments in $\groundext$ other than \newcasearg\ have outcome $o$.

\item If \(o = \defoutcome\), then let \(\caseii\) be an attacker of $\defcase$,  and thus of the form \(\caseii = \case{\charac\caseii}{\nondefoutcome}\) {(again see how regularity is necessary, since otherwise $\newcasearg$ could be the attacker)}. \(\caseii\) is not in $\groundext$ and, since $\groundext$ is also a stable extension, {some argument in} $\groundext$ attacks \(\caseii\). This is true for any attacker \(\caseii\) of the default argument, and thus the default argument is defended by $\groundext$. As $\groundext$ contains every argument it defends, the default argument is in the grounded extension, confirming that the outcome for $\newcasecharac$ is $\defoutcome$. \qedhere
\end{enumerate}
\end{proof}

\subsubsection{Spikes.} We are now ready to present a proof sketch of Theorem \ref*{theo:remove-spike} and a full proof of Theorem \ref*{theo:no-spikes}, the more interesting result.

\begin{proof}[Proof sketch of Theorem \ref*{theo:remove-spike}]
One may remember that the definition of the grounded extension is of the form $\groundext = \bigcup_{i \geqslant 0} G_i$ and verify by induction on $i$ that whether an argument is on the grounded extension or not is a function of whether the arguments that can reach it are or not in the grounded extension. This may be more clear using the method of labellings \cite{DBLP:books/sp/09/ModgilC09}. Thus whether $\defcase \in \groundext$ or not does not depend on a spike.
\end{proof}

\begin{proof}[Proof of Theorem \ref*{theo:no-spikes}]
    We will show by a contradiction. Let $\casei$ be a spike and assume $\casei \in \Args$. Then, during the execution of Algorithm \ref{algo:cumulaacbr}, at some moment $\casei$ was added to \texttt{to\_add}. Indeed, assume additionally that $\casei$ is the first such spike to be added.
    In order to this to happen, the predicted outcome at the moment for $\casei$ was not $\outcome{\casei}$. By (the contrapositive of) Theorem \ref{theo:nearest_neighbours}, this implies that there is a case $\caseii$ which is a nearest case to $\casei$ with the different outcome, that is, $\outcome{\caseii} \neq \outcome{\casei}$. It is straightforward to check that every nearest case of a new case is defended by it. Thus, by Lemma \ref{lemma:attacked-entering-case}, item 4, when $\casei$ is added to $\Args$, $\casei$ attacks $\caseii$.

    Now notice that, since Algorithm \ref{algo:cumulaacbr} goes by strata, and $\caseii$ is in the AF at this point of the algorithm, then it also is at the final AF, that is, $\caseii \in \Args$, as is the attack from $\casei$ to $\caseii$.
    However, since $\casei$ is a spike, it has no path to the default case, and thus $\caseii$ also has no path to the default case, that is, it is also a spike. This is absurd, since it contradicts our assumption that $\casei$ is the first spike to be added to $\Args$.
    Therefore, no spike is added to $\Args$.
\end{proof}

\section{Case study - Interpretation of factors}
\begin{table*}[htb]
  \centering
  \begin{tabular}{|c|p{0.90\textwidth}|}
    \hline
    Feature & Description \\ \hline
    $F^{1}_{\defendant}$ & plaintiff disclosed its product information in negotiations with defendant \\\hline
    $F^{2}_{\plaintiff}$ & defendant paid plaintiff's former employee to switch employment, apparently in an attempt to induce the employee to bring plaintiff's information \\\hline
    $F^{3}_{\defendant}$ & defendant's employee was the sole developer of plaintiff's product \\\hline
    $F^{4}_{\plaintiff}$ & defendant entered into a nondisclosure agreement with plaintiff \\\hline
    $F^{5}_{\defendant}$ & the nondisclosure agreement did not specify which information was to be treated as confidential \\\hline
    $F^{6}_{\plaintiff}$ & plaintiff took active measures to limit access to and distribution of its information \\\hline
    $F^{7}_{\plaintiff}$ & plaintiff's former employee brought product development information to defendant \\\hline
    $F^{8}_{\plaintiff}$ & defendant's access to plaintiff's product information saved it time or expense \\\hline
    $F^{10}_{\defendant}$ & plaintiff disclosed its product information to outsiders \\\hline
    $F^{11}_{\defendant}$ & plaintiff's information was about customers and suppliers (i.e. it may have been available independently from customers or even in directories) \\\hline
    $F^{12}_{\plaintiff}$ & plaintiff's disclosures to outsiders were subject to confidentiality restrictions \\\hline
    $F^{13}_{\plaintiff}$ & plaintiff and defendant entered into a noncompetition agreement \\\hline
    $F^{14}_{\plaintiff}$ & defendant used materials that were subject to confidentiality restrictions \\\hline
    $F^{15}_{\plaintiff}$ & plaintiff's information was unique in that plaintiff was the only manufacturer making the product \\\hline
    $F^{16}_{\defendant}$ & plaintiff's product information could be learned by reverse-engineering \\\hline
    $F^{17}_{\defendant}$ & defendant developed its product by independent research \\\hline
    $F^{18}_{\plaintiff}$ & defendant's product was identical to plaintiff's \\\hline
    $F^{19}_{\defendant}$ & plaintiff did not adopt any security measures \\\hline
    $F^{20}_{\defendant}$ & plaintiff's information was known to competitors \\\hline
    $F^{21}_{\plaintiff}$ & defendant obtained plaintiff's information altough he knew that plaintiff's information was confidential \\\hline
    $F^{22}_{\plaintiff}$ & defendant used invasive techniques to gain access to plaintiff's information \\\hline
    $F^{23}_{\defendant}$ & plaintiff entered into an agreement waiving confidentiality \\\hline
    $F^{24}_{\defendant}$ & the information could be obtained from publicly available sources \\\hline
    $F^{25}_{\defendant}$ & defendant discovered plaintiff's information through reverse engineering \\\hline
    $F^{26}_{\plaintiff}$ & defendant obtained plaintiff's information through deception \\\hline
    $F^{27}_{\defendant}$ & plaintiff disclosed its information in a public forum \\\hline
  \end{tabular}
  \caption{Meaning of factors.}
  \label{tab:case-study-interp}
\end{table*}

For ease of reference, we include here in Table \ref{tab:case-study-interp} the meaning of factors, as seen in \cite{Grabmair_thesis}.

\end{document}